%% file: main.tex
\documentclass{article} % For LaTeX2e
\usepackage{iclr2027_conference,times}

\usepackage{hyperref}
\usepackage{url}
\usepackage{graphicx}
\usepackage{booktabs}
\usepackage{multirow}
\usepackage{bm}

\usepackage{amsmath} 
\usepackage{wrapfig}
\usepackage{algorithm}
\usepackage{float}
\usepackage{algorithmic}
\usepackage{subcaption}
\usepackage{bbm}
\usepackage{amssymb}

\usepackage[utf8]{inputenc} % allow utf-8 input
\usepackage[T1]{fontenc}    % use 8-bit T1 fonts
\usepackage{hyperref}       % hyperlinks
\usepackage{url}            % simple URL typesetting
\usepackage{booktabs}       % professional-quality tables
\usepackage{amsfonts}       % blackboard math symbols
\usepackage{nicefrac}       % compact symbols for 1/2, etc.
\usepackage{microtype}      % microtypography
\usepackage{xcolor}         % colors

\usepackage{booktabs}
\usepackage{multirow}
\usepackage{graphicx}
\usepackage{bm}
\newtheorem{proposition}{Proposition}

\title{Learning beyond Site Bias for OOD Generalization in Brain Networks}

\author{Yingxu Wang\textsuperscript{1}\thanks{Equal contribution.}, Kunyu Zhang\textsuperscript{2}\footnotemark[1], Yanwu Yang\textsuperscript{3}, Mengzhu Wang\textsuperscript{4},
\\
\textbf{Thomas Wolfers\textsuperscript{3}, Yujie Wu\textsuperscript{5}, Siyang Gao\textsuperscript{6}, Nan Yin\textsuperscript{7}}
\\
\textsuperscript{1} Mohamed bin Zayed University of Artificial Intelligence \;\;
\textsuperscript{2} Zhengzhou University \\
\textsuperscript{3} University Hospital T\"{u}bingen \;\;
\textsuperscript{4} Hebei University of Technology \\
\textsuperscript{5} The Hong Kong Polytechnic University \;\;
\textsuperscript{6} City University of Hong Kong \\
\textsuperscript{7} The Education University of Hong Kong \\
\texttt{\{yingxv.wang,kunyu.zky,yangyanwu1111,dreamkily\}@gmail.com} \\
\texttt{\{dr.thomas.wolfers,yinnan8911\}@gmail.com} \\
\texttt{wu-yj16@tsinghua.org.cn, siyangao@cityu.edu.hk} }

\iclrfinalcopy

\def\method{CORE} 
\begin{document}

\maketitle

\begin{abstract}
Graph-based learning from functional magnetic resonance imaging (fMRI) has shown strong potential for brain network analysis. However, existing methods often degrade under cross-site out-of-distribution (OOD) settings, as site-dependent confounder effects can obscure disease-related connectivity patterns and static functional connectivity (FC) does not explicitly capture informative within-scan variations. In this paper, we propose \textbf{C}ross-site \textbf{O}OD \textbf{R}obust brain n\textbf{E}twork (\method{}), a unified framework for brain network learning across unseen sites. First, \method{} estimates site-specific confounder effects and aggregates the resulting deconfounders by cross-source reliability for unseen-site correction, while extracting a population scaffold of reproducible label-associated connections from site-wise residualized FC. It then summarizes temporal variations on scaffold edges into compact descriptors for line-graph modeling. Finally, prior-guided subject-adaptive gating modulates message passing, balancing population priors with individual variability. Extensive leave-one-site-out experiments on ABIDE, REST-meta-MDD, SRPBS, and ABCD demonstrate that \method{} consistently outperforms competitive baselines, with up to a 10.3\% relative improvement in accuracy. These gains also persist across different brain parcellation schemes on ABIDE.
\end{abstract}

\input{main/1_intro}
\input{main/2_related}

\input{main/4_method}
\input{main/5_experiments}
\input{main/6_conclusion}

\bibliography{reference}
\bibliographystyle{iclr2027_conference}

\appendix
\input{main/7_appendix}

\end{document}

%% file: main/1_intro.tex
\vspace{-0.3cm}
\section{Introduction}

\begin{wrapfigure}{r}{0.34\textwidth}
    \vspace{-0.65cm}
    \centering\includegraphics[width=\linewidth]{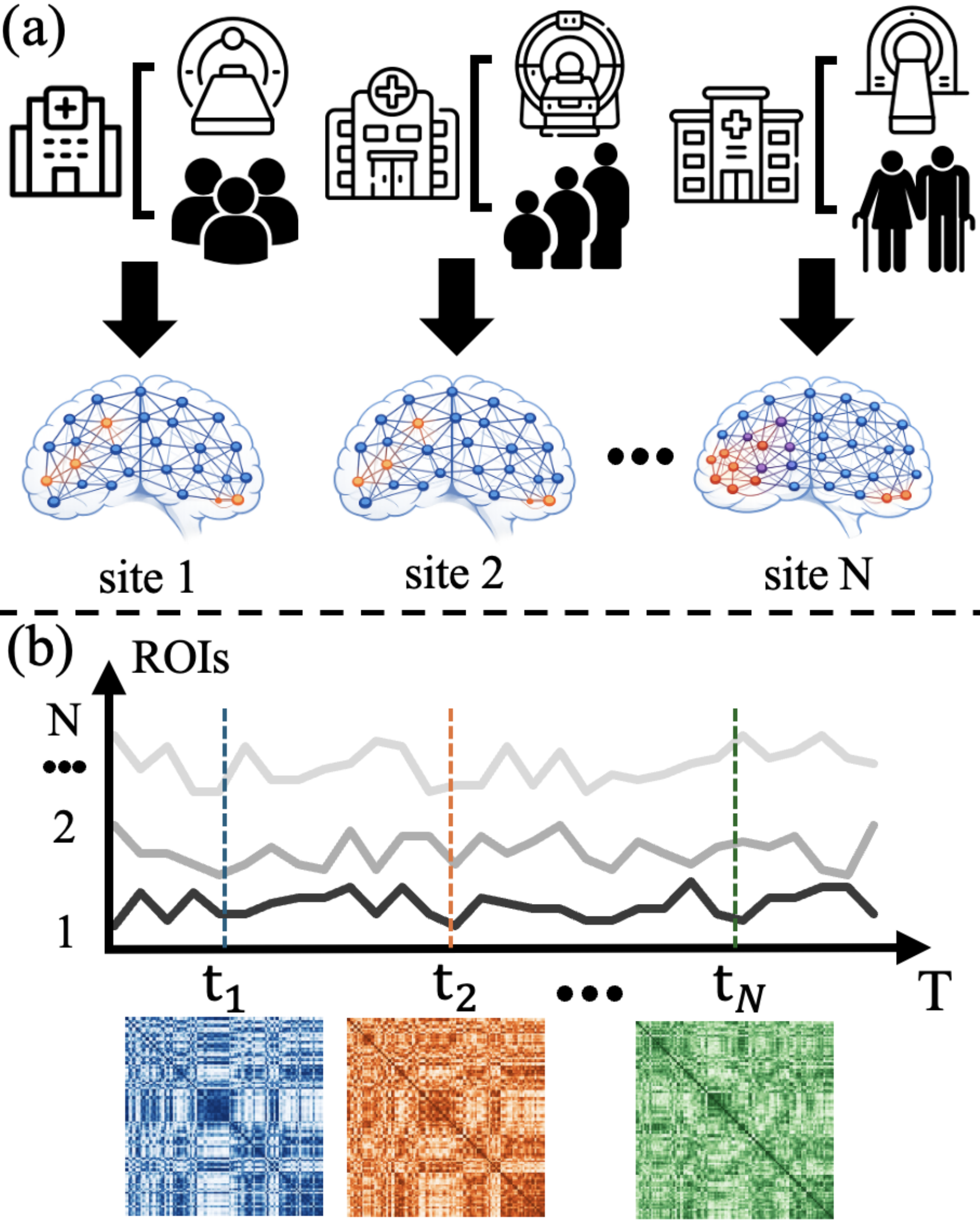}
    \vspace{-0.5cm}
    \caption{(a) Site-conditioned variations. (b) Time-varying FC.}
    \label{fig:figure1}
    \vspace{-0.6cm}
\end{wrapfigure}

Graph neural networks (GNNs) have shown strong potential for analyzing brain networks derived from functional magnetic resonance imaging (fMRI), with applications to autism spectrum disorder (ASD) and major depressive disorder (MDD)~\cite{bessadok2022graph,luo2024graph}. Conventional approaches estimate functional connectivity (FC) by quantifying statistical dependencies between blood oxygenation level-dependent (BOLD) time series from different regions of interest (ROIs)~\cite{finn2015functional,kawahara2017brainnetcnn}. The resulting FC graphs enable GNNs to model functional relationships among brain regions and learn disease-related representations for computer-aided diagnosis~\cite{gu2025fchgnn,cui2022interpretable}.

Despite this progress, existing methods often degrade under cross-site out-of-distribution (OOD) settings~\cite{qiu2025metaexplainer,qiu2024towards,zhang2026brainriem}. A key difficulty arises from structured, site-conditioned variations in FC associated with scanner characteristics and demographic composition~\cite{yamashita2025computational,yu2018statistical,wang2026nested}, as illustrated in Fig.~\ref{fig:figure1}(a). These effects can lead models to exploit site-related predictive shortcuts that may not remain reliable at unseen sites~\cite{brown2023detecting,chen2022mitigating,wang2025protomol}. At the same time, static FC estimated over an entire scan does not explicitly capture within-scan connectivity variations that can carry complementary disease-related information~\cite{allen2014tracking,kaiser2016dynamic} (Fig.~\ref{fig:figure1}(b)). Although dynamic FC modeling incorporates these temporal variations, the underlying connectivity estimates remain sensitive to acquisition and sampling differences across sites~\cite{huotari2019sampling,james2019impact,wang2026sgac}, complicating the identification of transferable connectivity patterns.

In this paper, we revisit cross-site OOD generalization for brain network analysis and aim to develop a unified framework. However, designing such a framework poses three key challenges:
(1) \textit{How can site-conditioned confounder bias be mitigated?}
Pooled deconfounding with shared coefficients assumes site-invariant covariate effects, although the association between a given covariate and FC can differ across sites~\cite{xu2025brainood,qiu2025metaexplainer}. For instance, when the association between age and a fronto-parietal FC edge varies across sites, subtracting a shared age effect can leave site-dependent residual bias and obscure reproducible disease-related connectivity patterns.
(2) \textit{How can dynamic FC be summarized compactly across scans of unequal length?}
Dynamic FC is represented by high-dimensional connectivity sequences whose lengths and temporal resolution vary with scan duration and sampling protocols~\cite{chang2026neurocognitive,huotari2019sampling}. These differences make it difficult to obtain comparable temporal representations across subjects and sites while retaining informative within-scan variations.
(3) \textit{How can cross-site regularities be reconciled with subject-specific variability?}
Disease-related connectivity patterns that are reproducible across sites provide population-level guidance. However, FC varies across individuals~\cite{mueller2013individual,yang2024brainmass}, and the same patterns may not be equally informative for every subject. Models based solely on population patterns may overlook informative individual differences, whereas subject-specific modeling alone may retain site-related variations.

To tackle these challenges, we propose \textbf{C}ross-site \textbf{O}OD \textbf{R}obust brain n\textbf{E}twork (\method{}), a unified framework with three components:
(i) \textit{Site-aware confounder decoupling} estimates edge-wise confounder effects at each source site via Huber regression and aggregates deconfounders using cross-source reliability weights for unseen-site correction. From site-wise residualized FC, we extract a population scaffold of reproducible label-associated connections using consensus contrast magnitudes, directional consistency, and bootstrap stability.
(ii) \textit{Transient pathway profiling} summarizes sliding-window FC trajectories on scaffold edges into fixed-dimensional temporal descriptors, whose dimension is independent of scan length. These descriptors and deconfounded static FC form line-graph node features, with nodes representing scaffold connections linked through shared ROIs.
(iii) \textit{Prior-guided subject-adaptive gating} combines scaffold-derived population priors with each subject's local pathway features and global context to compute adaptive gates. These gates modulate line-graph message passing and readout, adapting pathway contributions to individual predictions. Leave-one-site-out (LOSO) evaluation on real-world benchmarks (ABIDE, REST-meta-MDD, SRPBS, and ABCD) shows that \method{} outperforms competitive baselines, with up to a 10.3\% relative gain in accuracy. Moreover, \method{} remains robust to atlas variations, maintaining improvements across parcellation schemes.

Our contributions are summarized as follows: (1) We revisit cross-site OOD generalization in brain networks by jointly considering site-conditioned confounding, compact dynamic FC representations, and the balance between cross-site regularities and subject-specific variability. (2) We propose \method{}, a unified framework that extracts a population scaffold from site-wise deconfounded FC to organize temporal pathway representations and guide subject-adaptive graph learning. (3) We conduct experiments on multiple real-world benchmarks, showing that \method{} outperforms baseline methods under cross-site OOD settings and remains robust across different brain parcellation schemes.

%% file: main/2_related.tex
\vspace{-0.cm}
\section{Related Work}
\vspace{-0.1cm}

\textbf{Graph Learning for Brain Network Analysis.}
Graph learning captures functional relationships among brain regions for subject-level prediction~\cite{li2021braingnn,thapaliya2025brainrgin,gu2025fchgnn,zhang2025mvho,wang2026riemannian}. Static FC approaches estimate pairwise dependencies over full BOLD sequences, representing whole-brain connectivity without explicitly capturing within-scan variations. Dynamic FC and spatio-temporal graph models extend these representations by incorporating connectivity changes and temporal dependencies~\cite{hu2024spatio,kong2021spatio}. However, full dynamic FC trajectories are high-dimensional, and the underlying connectivity estimates are sensitive to acquisition and analysis choices, including sampling rate and windowing~\cite{james2019impact,preti2017dynamic,huotari2019sampling}. Cross-site generalization further requires accounting for FC variations associated with acquisition protocols and population composition~\cite{qiu2025metaexplainer,yamashita2025computational,zhang2026modeling,wang2026usbd}. In this context, \method{} summarizes temporal variations on a population scaffold of reproducible label-associated connections into compact descriptors for subject-adaptive pathway modeling.

\textbf{OOD Generalization on Graphs.}
Graph OOD generalization seeks robustness to distribution shifts through invariant representation learning, spurious subgraph pruning, and information bottleneck objectives~\cite{yao2024empowering,li2025out,yao2025pruning,yao2025learning}. These strategies have been explored in molecular and social graphs~\cite{zhang2026survey,cai2025out}. Recent work on brain-network OOD learning combines feature filtering with consistent structure selection across subjects to identify predictive connections in static FC~\cite{xu2025brainood,zhao2026unified}. Complementary neuroimaging research uses statistical harmonization to adjust site effects in FC measurements while accounting for observed covariates~\cite{yu2018statistical}. For unseen-site prediction, \method{} estimates covariate effects separately at each source site and extracts a population scaffold of reproducible label-associated connections from residualized FC to guide subject-adaptive graph learning.

%% file: main/4_method.tex
\vspace{-0.2cm}
\section{Methodology}

\begin{figure}
    \centering
    \includegraphics[width=1.00\linewidth]{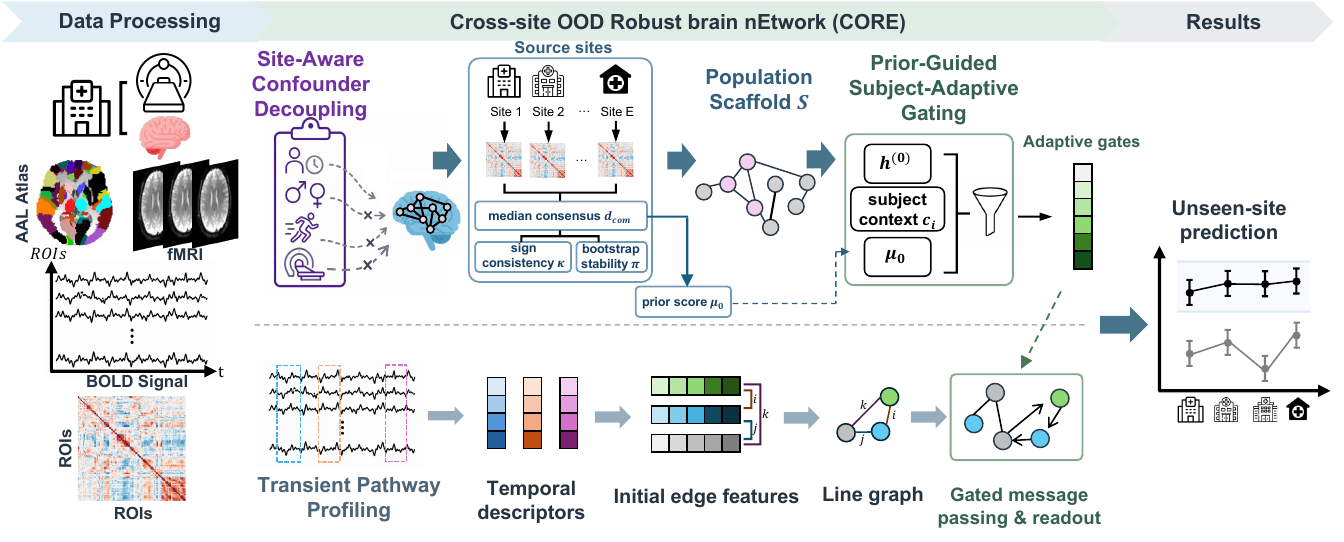}
    \vspace{-0.7cm}
    \caption{Overview of the proposed \method{}. Site-Aware Confounder Decoupling mitigates site-conditioned bias and extracts a population scaffold of reproducible label-associated connections. Transient Pathway Profiling summarizes temporal variations on scaffold edges into lightweight descriptors and represents these connections as line-graph nodes. Prior-Guided Subject-Adaptive Gating combines scaffold priors with subject-specific features and context to modulate message passing for cross-site prediction.}
    \label{fig:framework}
    \vspace{-0.5cm}
\end{figure}

\textbf{Problem Setup.}
We study cross-site OOD generalization for brain network analysis, treating imaging sites as distinct environments. Let $\mathcal{E}_{\mathrm{all}}$ denote all environments, with training environments $\mathcal{E}_{\mathrm{tr}}\subset\mathcal{E}_{\mathrm{all}}$ and unseen test environments $\mathcal{E}_{\mathrm{te}}\subseteq\mathcal{E}_{\mathrm{all}}\setminus\mathcal{E}_{\mathrm{tr}}$. Each environment $e\in\mathcal{E}_{\mathrm{all}}$ is associated with a distribution $\mathbb{P}^e$ over $(\bm X,\bm q,Y)$. For each $e\in\mathcal{E}_{\mathrm{tr}}$, we observe a dataset $\widehat{\mathcal{D}}^e=\{(\bm X_i,\bm q_i,Y_i)\}_{i=1}^{N_e}$ of $N_e$ subjects sampled from $\mathbb{P}^e$. Here, $\bm X_i\in\mathbb{R}^{T_i\times P}$ denotes the BOLD time series of subject $i$ over $P$ ROIs, with sequence length $T_i$ varying across subjects. The vector $\bm q_i\in\mathbb R^d$ contains within-site covariates (age, sex, head motion) and site-level covariates (sequence length and TR), all available at inference. The binary label $Y_i\in\{0,1\}$ indicates the control ($0$) or condition ($1$) group.

We consider predictors $f=\rho\circ T_{\widehat{\Theta}}$, where $T_{\widehat{\Theta}}(\bm X,\bm q)$ maps BOLD time series and structured covariates to graph-level representations, $\widehat{\Theta}$ denotes the parameters and statistics estimated from training data, and $\rho$ is the classifier. Our goal is to generalize from $\mathcal{E}_{\mathrm{tr}}$ to unseen environments $\mathcal{E}_{\mathrm{te}}$. We define the OOD risk as the average prediction risk across these environments:
\begin{equation}
    \mathcal{R}_{\mathrm{OOD}}(f)
    =\sum_{e\in\mathcal{E}_{\mathrm{te}}}
    \frac{1}{|\mathcal{E}_{\mathrm{te}}|} \mathbb{E}_{(\bm X,\bm q,Y)\sim\mathbb{P}^e}
    \left[
    l\big(f(\bm X,\bm q),Y\big)
    \right],
\end{equation}
where $l(\cdot,\cdot)$ is the prediction loss.

\textbf{Overview.} We propose \method{}, a unified framework for cross-site OOD generalization in brain network analysis (Fig.~\ref{fig:framework}). It comprises three components: (i) \textbf{Site-Aware Confounder Decoupling} (Sec.~\ref{sec:decoupling}) estimates site-specific deconfounders, aggregates them using cross-source reliability for unseen-site correction, and extracts a population scaffold of reproducible label-associated edges from site-wise residualized FC. (ii) \textbf{Transient Pathway Profiling} (Sec.~\ref{sec:profiling}) summarizes sliding-window FC trajectories on scaffold edges into fixed-dimensional temporal descriptors and represents these edges as line-graph nodes linked by shared ROIs. (iii) \textbf{Prior-Guided Subject-Adaptive Gating} (Sec.~\ref{sec:gating}) combines scaffold-derived priors with subject-specific pathway features and context to modulate line-graph message passing and readout, producing subject-level representations.

\subsection{Site-Aware Confounder Decoupling}
\label{sec:decoupling}

Associations between observed covariates and FC can vary across imaging sites. We first characterize the residual bias that can arise when these associations are modeled using a shared linear deconfounding.

Let $\bm r_i\in\mathbb R^M$ denote the static FC vector of subject $i$ from site $e$, obtained by stacking the upper-triangular entries (excluding the diagonal) of the full-scan ROI-wise Pearson correlation matrix in a fixed order, where $M=P(P-1)/2$. The edge index set is $\mathcal I=\{1,\dots,M\}$, with each $j\in\mathcal I$ corresponding to an undirected ROI pair $(u_j,v_j)$ satisfying $1\le u_j<v_j\le P$. We standardize the observed covariates as $\tilde{\bm q}_i=(\bm q_i-\bm\mu_q)/\bm\sigma_q$, where $\bm\mu_q,\bm\sigma_q\in\mathbb R^d$ are the covariate means and standard deviations estimated from pooled training subjects, and division is element-wise. A pooled linear deconfounding model computes
\begin{equation}
    \hat{\bm{r}}_i^{\mathrm{glob}}
    =
    \bm{r}_i
    -
    \bm{b}_{\mathrm{glob}}
    -
    \bm{\Gamma}_{\mathrm{glob}}^{\top}\tilde{\bm{q}}_i,
\end{equation}
where $\bm b_{\mathrm{glob}}\in\mathbb R^M$ and
$\bm\Gamma_{\mathrm{glob}}\in\mathbb R^{d\times M}$
are the edge-wise intercept vector and covariate coefficient matrix
shared across sites.

\vspace{-0.2cm}
\begin{proposition}[Residual Bias of Pooled Deconfounding]
\label{pro1}
Assume FC follows the site-specific additive model
\begin{equation}
    \bm{r}_i
    =
    \bm{b}_e^{\star}
    +
    \bm{\Gamma}_e^{\star\top}\tilde{\bm{q}}_i
    +
    \bm{s}_i^{\star}
    +
    \bm{\epsilon}_i,
\end{equation}
where $\bm{b}_e^{\star}\in\mathbb{R}^{M}$ and $\bm{\Gamma}_e^{\star}\in\mathbb{R}^{d\times M}$ are the site-specific intercept and covariate coefficients, $\bm{s}_i^{\star}\in\mathbb{R}^{M}$ is the confounder-free FC component with $\mathbb E[\bm s_i^{\star}\mid e,\tilde{\bm q}_i,Y_i]=\mathbb E[\bm s_i^{\star}\mid e,Y_i]$, and $\mathbb{E}[\bm{\epsilon}_i\mid e,\tilde{\bm{q}}_i,\bm{s}_i^{\star},Y_i]=\bm{0}$. Let $\mathrm D_e[\bm v]=\mathbb E[\bm v\mid e,Y_i=1]-\mathbb E[\bm v\mid e,Y_i=0]$ denote the diagnostic contrast at site $e$. Then, for any fixed global coefficients $(\bm{b}_{\mathrm{glob}},\bm{\Gamma}_{\mathrm{glob}})$,
\begin{equation}
    \mathrm D_e\big[\hat{\bm r}_i^{\mathrm{glob}}\big]
    =
    \mathrm D_e[\bm s_i^{\star}]
    +
    (\bm{\Gamma}_e^{\star}-\bm{\Gamma}_{\mathrm{glob}})^{\top}
    \mathrm D_e[\tilde{\bm q}_i],
\end{equation}
whereas the site-wise residual $\bm r_i-\bm b_e^{\star}-\bm\Gamma_e^{\star\top}\tilde{\bm q}_i$ satisfies $\mathrm D_e[\cdot]=\mathrm D_e[\bm s_i^{\star}]$.
\end{proposition}
\vspace{-0.2cm}

Proposition~\ref{pro1} shows that pooled deconfounding contaminates the site-level diagnostic contrast whenever covariate effects differ across sites and the two groups differ in within-site covariates, so edges may be selected for reasons unrelated to diagnosis; site-wise deconfounding removes this term.

\textbf{(i) Site-Aware Deconfounding.}
We estimate covariate-associated FC trends separately at each source site using Huber regression, which limits the influence of large residuals. For each training site $e\in\mathcal E_{\mathrm{tr}}$ and edge $j\in\mathcal I$, we solve the following problem using iteratively reweighted least squares (IRLS)~\cite{fox2002robust,li1998linear,holland1977robust}:
\vspace{-0.1cm}
\begin{equation}
\label{eq:7}
    (\hat b_{e,j}, \hat{\bm{\gamma}}_{e,j}, \hat\theta_{e,j})
    =
    \arg\min_{b,\bm{\gamma},\theta}
    \sum_{i=1}^{N_e}
    \mathcal{H}_{\delta}\!\left(r_{i,j} - b - \tilde{\bm{q}}_i^{\top}\bm{\gamma} - \theta Y_i\right),
\end{equation}
where $\hat\theta_{e,j}$ absorbs the diagnostic mean difference so that $(\hat b_{e,j},\hat{\bm\gamma}_{e,j})$ estimates the covariate effect without removing the diagnostic contrast; only the covariate part is subtracted in Eq.~\eqref{eq:site_res}, so no label is needed at inference. Stacking the estimates across
edges gives
$
    \hat{\bm b}_e
    =
    [\hat b_{e,1},\dots,\hat b_{e,M}]^\top
    \in\mathbb R^M
$
and
$
    \hat{\bm\Gamma}_e
    =
    [\hat{\bm\gamma}_{e,1},\dots,\hat{\bm\gamma}_{e,M}]
    \in\mathbb R^{d\times M}.
$
For scaffold estimation, we subtract the fitted site-specific trends
from each source subject's FC:
\begin{equation}
\label{eq:site_res}
    \bm{r}_i^{\mathrm{site}}
    =
    \bm{r}_i
    -
    \hat{\bm{b}}_{e}
    -
    \hat{\bm{\Gamma}}_{e}^{\top}\tilde{\bm{q}}_i,
    \qquad i=1,\dots,N_e.
\end{equation}

For unseen-site prediction, we aggregate source-site deconfounders according to their agreement on source covariates. Let $\mathcal C_e(\tilde{\bm q})=\hat{\bm b}_e+\hat{\bm\Gamma}_e^\top\tilde{\bm q}$ denote the deconfounding map of site $e$. Its cross-source disagreement and normalized reliability weight are
\begin{equation}
\label{eq:source_disagreement}
u_e
=
\frac{1}{|\mathcal E_{\mathrm{tr}}|-1}
\sum_{\substack{e'\in\mathcal E_{\mathrm{tr}}\\e'\neq e}}
\frac{1}{N_{e'}M}
\sum_{i=1}^{N_{e'}}
\left\|
\mathcal C_e(\tilde{\bm q}_i)-\mathcal C_{e'}(\tilde{\bm q}_i)
\right\|_2^2,
\qquad
\alpha_e
=
\frac{(u_e+\varepsilon_u)^{-1}}
{\sum_{k\in\mathcal E_{\mathrm{tr}}}(u_k+\varepsilon_u)^{-1}},
\end{equation}
where the inner sum runs over subjects at site $e'$ and $\varepsilon_u>0$ ensures numerical stability. Smaller $u_e$ yields a larger weight. We then form a frozen aggregate deconfounder and apply it to each unseen subject $i$ using covariates standardized with the training statistics:
\begin{equation}
\label{eq:weighted_deconfounder}
\bar{\bm b}_w
=
\sum_{e\in\mathcal E_{\mathrm{tr}}}\alpha_e\hat{\bm b}_e,
\qquad
\bar{\bm\Gamma}_w
=
\sum_{e\in\mathcal E_{\mathrm{tr}}}\alpha_e\hat{\bm\Gamma}_e,
\qquad
\bm r_i^{\mathrm{res}}
=
\bm r_i-\bar{\bm b}_w-\bar{\bm\Gamma}_w^\top\tilde{\bm q}_i.
\end{equation}

At an unseen site, this aggregate extrapolates source-site intercepts to its acquisition values and inherits a residual term as in Proposition~\ref{pro1}, reduced but not removed by reliability weighting. Site-level covariates act only here, not in $\bm r_i^{\mathrm{site}}$ or the scaffold. As residual FC is high-dimensional and unstable, we extract a population scaffold from site-wise residualized source data.

\textbf{(ii) Cross-Site Scaffold Extraction.}
Using the site-wise residualized training FC $\bm r_i^{\mathrm{site}}$,
we extract a population scaffold $\mathcal S\subseteq\mathcal I$
of connections with reproducible label associations across source sites.
For each source site $e$, let $\mathcal D_e^{(y)}$ index subjects
with label $y\in\{0,1\}$, with both groups represented.
We define the site-level diagnostic contrast as
$
    \bm d_e
    =
    \bm\mu_e^{(1)}-\bm\mu_e^{(0)},
$
where $\bm\mu_e^{(1)},\bm\mu_e^{(0)}\in\mathbb R^M$
are coordinate-wise Huber location estimates of
$\bm r_i^{\mathrm{site}}$ in the case and control groups, respectively.

To obtain a site-balanced consensus contrast, we aggregate
$\{\bm d_e\}_{e\in\mathcal E_{\mathrm{tr}}}$
using a fixed component-wise median rule:
\begin{equation}
    \bm d_{\mathrm{com}}
    \in
    \arg\min_{\bm u\in\mathbb R^M}
    \sum_{e\in\mathcal E_{\mathrm{tr}}}
    \|\bm d_e-\bm u\|_1.
\end{equation}
Each source site contributes equally, regardless of sample size. Huber location thresholds are fixed, and nonunique Huber location and median minimizers are resolved by midpoint selection. We assess directional consistency across source sites using $\kappa_j$ and the stability of the consensus sign under site resampling using $\pi_j$:
\vspace{-0.3cm}
\begin{equation}
\label{eq:11}
\kappa_j =
\frac{1}{|\mathcal E_{\mathrm{tr}}|}
\sum_{e\in\mathcal E_{\mathrm{tr}}}
\mathbb I\!\left[
\operatorname{sgn}(d_{e,j})
=
\operatorname{sgn}(d_{\mathrm{com},j})
\right],
\quad
\pi_j =
\frac{1}{B}
\sum_{b=1}^{B}
\mathbb I\!\left[
\operatorname{sgn}(d^{(b)}_{\mathrm{com},j})
=
\operatorname{sgn}(d_{\mathrm{com},j})
\right].
\end{equation}
For each of the $B$ bootstrap replicates, we sample
$|\mathcal E_{\mathrm{tr}}|$ source sites uniformly with replacement
and recompute the component-wise median of the sampled site contrasts,
retaining repeated draws. The resulting contrast for edge $j$ is
$d_{\mathrm{com},j}^{(b)}$, and we define
$\operatorname{sgn}(0)=0$.
The population scaffold retains connections satisfying all three
selection criteria:
\begin{equation}
\label{eq:12}
    \mathcal S
    =
    \left\{
    j\in\mathcal I
    \;\middle|\;
    |d_{\mathrm{com},j}|>\tau_E,\;
    \kappa_j\ge\eta_E,\;
    \pi_j\ge\zeta_E
    \right\},
\end{equation}
where $\tau_E$, $\eta_E$, and $\zeta_E$ are selected on validation sites and all scaffold statistics are from training sites.

\vspace{-0.2cm}
\begin{proposition}[Margin Stability of the Site-Aware Scaffold]
\label{prop:shortcut_suppression}
For edge $j\in\mathcal I$, let $r^{\mathrm{site}}_{i,j}=r_{i,j}-\hat b_{e,j}-\tilde{\bm q}_i^{\top}\hat{\bm\gamma}_{e,j}$ and $r^{\star}_{i,j}=r_{i,j}-b^{\star}_{e,j}-\tilde{\bm q}_i^{\top}\bm\gamma^{\star}_{e,j}$ denote the estimated and oracle site-wise residuals. Let $\mathcal S^{\star}$ be the scaffold obtained by applying the same selection rule, thresholds, and bootstrap draws to the oracle residuals, with oracle contrasts $d^{\star}_{e,j}$, $d^{\star}_{\mathrm{com},j}$, and $d^{\star(b)}_{\mathrm{com},j}$. Define the estimation perturbation and the oracle margin of edge $j$ as
\begin{gather}
    \Delta_j
    =
    \max_{e\in\mathcal E_{\mathrm{tr}}}
    \sum_{y\in\{0,1\}}
    \sup_{i\in\mathcal D_e^{(y)}}
    \big|r^{\mathrm{site}}_{i,j}-r^{\star}_{i,j}\big|,
    \\
    m^{\star}_j
    =
    \min\Big\{
    \min_{e\in\mathcal E_{\mathrm{tr}}}|d^{\star}_{e,j}|,\;
    \min_{1\le b\le B}|d^{\star(b)}_{\mathrm{com},j}|,\;
    |d^{\star}_{\mathrm{com},j}|,\;
    \big||d^{\star}_{\mathrm{com},j}|-\tau_E\big|
    \Big\}.
\end{gather}
If $\Delta_j<m^{\star}_j$, then $j\in\mathcal S$ if and only if $j\in\mathcal S^{\star}$.
\end{proposition}
\vspace{-0.2cm}

Proposition~\ref{prop:shortcut_suppression} shows that an edge retains its scaffold membership whenever the estimation error of the site-wise deconfounders is smaller than the smallest margin the selection rule depends on. Under the model of Proposition~\ref{pro1}, this error is variance rather than bias and shrinks with site size, so edges with non-negligible oracle margins are selected consistently.

\subsection{Transient Pathway Profiling}
\label{sec:profiling}

Static FC estimated over an entire scan does not explicitly capture within-scan connectivity variation. Full dynamic FC trajectories are high-dimensional and vary in length, complicating comparisons across subjects and sites~\cite{chang2026neurocognitive,huotari2019sampling}. We therefore summarize temporal variations on the population scaffold $\mathcal S$ into fixed-dimensional descriptors that complement static FC.

Let $M_{\mathcal S}=|\mathcal S|$ and let $(j_p)_{p=1}^{M_{\mathcal S}}$ denote a fixed enumeration of the scaffold edges. For each subject $i$, we extract sliding windows of length $W$ and stride $S_w$ from $\bm X_i$, yielding $L_i=\left\lfloor(T_i-W)/S_w\right\rfloor+1$ windows, where $W$ and $S_w$ are measured in time points. For each edge $j\in\mathcal S$ and window $t\in\{1,\dots,L_i\}$, let $\bm x_{i,u_j}^{(t)},\bm x_{i,v_j}^{(t)}\in\mathbb R^W$ denote the windowed signals from its two ROIs. We compute the windowed Pearson correlation $\rho_{i,j}^{(t)}=\operatorname{corr}(\bm x_{i,u_j}^{(t)},\bm x_{i,v_j}^{(t)})$ and apply the Fisher-$z$ transformation $c_{i,j}^{(t)}=\frac{1}{2}\log\frac{1+\rho_{i,j}^{(t)}}{1-\rho_{i,j}^{(t)}}$.

We use temporal volatility $s_{i,j}$ to summarize dispersion across windows and temporal range $f_{i,j}$ to capture the full span of connectivity variation. Let $\bar c_{i,j}=\frac{1}{L_i}\sum_{t=1}^{L_i}c_{i,j}^{(t)}$ denote the temporal mean. The two descriptors are defined as
\begin{equation}
\label{eq:descriptor}
    s_{i,j}
    =
    \sqrt{
        \frac{1}{L_i}
        \sum_{t=1}^{L_i}
        \left(c_{i,j}^{(t)}-\bar c_{i,j}\right)^2
    },
    \quad
    f_{i,j}
    =
    \max_{1\le t\le L_i}c_{i,j}^{(t)}
    -
    \min_{1\le t\le L_i}c_{i,j}^{(t)}.
\end{equation}
Both descriptors are invariant to time-constant additive shifts in the Fisher-$z$-transformed trajectories. Let $\bm r_i^{\mathrm{dec}}$ denote the deconfounded static FC, equal to $\bm r_i^{\mathrm{site}}$ for training-site subjects, whose own site offsets are removed exactly, and to $\bm r_i^{\mathrm{res}}$ otherwise, which removes the source-shared covariate effects but not the held-out site's own offset. For each scaffold edge $j_p$, we combine its deconfounded static FC with the log-transformed temporal descriptors to obtain the initial pathway representation:
\begin{equation}
\label{eq:node_feature}
    \bm h_{i,p}^{(0)}
    =
    \left[
        r_{i,j_p}^{\mathrm{dec}},
        \log(s_{i,j_p}+\varepsilon),
        \log(f_{i,j_p}+\varepsilon)
    \right]^\top
    \in\mathbb R^3,
\end{equation}
where $\varepsilon>0$ ensures numerical stability.

To model interactions among pathways sharing brain regions, we construct a shared line graph $\mathcal G_{\mathrm{line}}=(\mathcal V,\mathcal E_{\mathrm{line}},\bm A)$, where $\mathcal V=\{1,\dots,M_{\mathcal S}\}$. Each node $p$ represents scaffold edge $j_p$ and carries subject-specific features $\bm h_{i,p}^{(0)}$. Distinct nodes $p$ and $p'$ are connected when their ROI pairs share an endpoint, i.e., $\{u_{j_p},v_{j_p}\}\cap\{u_{j_{p'}},v_{j_{p'}}\}\neq\emptyset$.

Line-graph edges receive larger weights when the corresponding pathways have similar absolute consensus contrasts. Let $a_p=|d_{\mathrm{com},j_p}|$ and $\bm a=[a_1,\dots,a_{M_{\mathcal S}}]^\top$. We standardize these magnitudes to obtain population prior scores $\bm\mu_0=(\bm a-\bar a\bm 1)/(\sigma_a+\varepsilon_\sigma)$, where
\begin{equation}
    \bar a
    =
    \frac{1}{M_{\mathcal S}}
    \sum_{p=1}^{M_{\mathcal S}}a_p,
    \qquad
    \sigma_a
    =
    \sqrt{
        \frac{1}{M_{\mathcal S}}
        \sum_{p=1}^{M_{\mathcal S}}(a_p-\bar a)^2
    },
\end{equation}
and $\varepsilon_\sigma>0$ ensures numerical stability. The weighted adjacency matrix $\bm A\in\mathbb R^{M_{\mathcal S}\times M_{\mathcal S}}$ is defined as
\begin{equation}
\label{eq:linegraph}
    A_{pp'}
    =
    \begin{cases}
        \exp\!\left(
            -\dfrac{|\mu_{0,p}-\mu_{0,p'}|}
            {\Delta_{\mu_0}+\varepsilon_A}
        \right),
        & (p,p')\in\mathcal E_{\mathrm{line}},\\
        0, & \text{otherwise},
    \end{cases}
\end{equation}
where $\Delta_{\mu_0}=\max_p\mu_{0,p}-\min_p\mu_{0,p}$ and $\varepsilon_A>0$ prevents division by zero. Finally, we add self-loops to obtain $\hat{\bm A}=\bm A+\bm I$ and compute the symmetrically normalized propagation matrix $\tilde{\bm A}=\hat{\bm D}^{-1/2}\hat{\bm A}\hat{\bm D}^{-1/2}$, where $\hat{\bm D}$ is diagonal with $\hat D_{pp}=\sum_{p'\in\mathcal V}\hat A_{pp'}$.

\subsection{Prior-Guided Subject-Adaptive Gating}
\label{sec:gating}

The shared line graph provides a population-level structure for pathway modeling, while functional connectivity varies across individuals~\cite{mueller2013individual,yang2024brainmass}. The same pathways may not be equally informative for every subject. We therefore combine scaffold-derived priors with subject-specific features to adapt pathway contributions during message passing and readout.

Given the initial node features $\{\bm h_{i,p}^{(0)}\}_{p=1}^{M_{\mathcal S}}$, we compute a subject-level context by pooling the scaffold-node features:
\begin{equation}
\label{eq:16}
    \bar{\bm h}_i^{(0)}
    =
    \frac{1}{M_{\mathcal S}}
    \sum_{p=1}^{M_{\mathcal S}}\bm h_{i,p}^{(0)},
    \qquad
    \bm c_i
    =
    \phi\!\left(\bar{\bm h}_i^{(0)}\right),
\end{equation}
where $\phi(\cdot)$ is a multi-layer perceptron (MLP). For each node $p\in\mathcal V$, we combine its local features with the subject context and add the population prior score $\mu_{0,p}$ to the resulting logit:
\begin{equation}
\label{eq:gate}
    \ell_{i,p}
    =
    \psi\!\left(
        \left[
            \bm h_{i,p}^{(0)}
            \,\|\,
            \bm c_i
        \right]
    \right)
    +
    \lambda\mu_{0,p},
    \qquad
    g_{i,p}
    =
    \sigma\!\left(\frac{\ell_{i,p}}{\tau}\right),
\end{equation}
where $\psi(\cdot)$ is an MLP producing a scalar output, $\|$ denotes concatenation, and $\sigma(\cdot)$ is the sigmoid function. The parameters $\lambda\ge0$ and $\tau>0$ control the prior strength and gate temperature, respectively. The resulting gate $g_{i,p}\in(0,1)$ combines population guidance with the subject's local and global connectivity features.

For each subject, the gates are computed from the initial features and reused across $L$ message-passing layers. Each gate scales the messages sent by its corresponding pathway, allowing subjects to assign different pathway weights within the shared propagation structure. At layer $l$, node $p$ is updated as
\begin{equation}
\label{eq:graph_update}
    \bm h_{i,p}^{(l)}
    =
    \bm h_{i,p}^{(l-1)}
    +
    \mathcal F^{(l)}\!\left(
        \bm h_{i,p}^{(l-1)},
        \sum_{p'\in\mathcal V}
        \tilde A_{pp'}\,g_{i,p'}\,\bm h_{i,p'}^{(l-1)}
    \right),
\end{equation}
where $\mathcal F^{(l)}(\cdot)$ is a learnable update function. Finally, we use the same gates to aggregate the refined node representations into a subject-level representation:
\begin{equation}
\label{eq:graph_level}
    \bm z_i
    =
    \frac{
        \sum_{p=1}^{M_{\mathcal S}}
        g_{i,p}\bm h_{i,p}^{(L)}
    }{
        \sum_{p=1}^{M_{\mathcal S}}g_{i,p}
        +
        \varepsilon_z
    },
\end{equation}
where $\varepsilon_z>0$ ensures numerical stability.

\subsection{Learning Objective}
\label{sec:objective}

Given the subject-level representation $\bm z_i$, we obtain the classifier output $\hat y_i=\rho(\bm z_i)$ using an MLP classifier $\rho$. Let $\widehat{\mathcal D}_{\mathrm{tr}}=\bigcup_{e\in\mathcal E_{\mathrm{tr}}}\widehat{\mathcal D}^{e}$ denote the pooled empirical training set. The classification loss is
\begin{equation}
    \mathcal L_{\mathrm{cls}}
    =
    \mathbb E_{(\bm X_i,\bm q_i,Y_i)\sim\widehat{\mathcal D}_{\mathrm{tr}}}
    \big[
        \ell_{\mathrm{CE}}(\hat y_i,Y_i)
    \big],
\end{equation}
where $\ell_{\mathrm{CE}}$ denotes cross-entropy loss.

To control the total pathway weight assigned to each subject, we impose a soft gate-budget regularizer that penalizes deviations from a target gate mass:
\begin{equation}
    \mathcal L_{\mathrm{budget}}
    =
    \mathbb E_{(\bm X_i,\bm q_i,Y_i)\sim\widehat{\mathcal D}_{\mathrm{tr}}}
    \left[
        \left|
            \sum_{p=1}^{M_{\mathcal S}}g_{i,p}
            -
            K
        \right|
    \right],
\end{equation}
where $K\in(0,M_{\mathcal S}]$ is the target aggregate gate mass. The overall training objective is
\begin{equation}
\label{eq:final}
    \mathcal L
    =
    \mathcal L_{\mathrm{cls}}
    +
    \beta\mathcal L_{\mathrm{budget}},
\end{equation}
where $\beta\ge0$ controls the gate-budget penalty. 

%% file: main/5_experiments.tex
\section{Experiments}

\subsection{Experimental Settings}

\textbf{Datasets.} To evaluate \method{} under cross-site OOD settings, we use four real-world fMRI datasets spanning three disorders: ABIDE~\cite{abide} for ASD, REST-meta-MDD~\cite{restmeatmdd} and SRPBS~\cite{srpbs} for MDD, and the large-scale developmental cohort ABCD~\cite{abcd} for ADHD. These datasets exhibit substantial variation in demographics and acquisition protocols across sites and, for SRPBS, across acquisition batches. In addition to the standard ABIDE setting, we construct an ABIDE (CC200) variant using the CC200 functional parcellation to assess robustness to atlas variation. Further dataset details are provided in Appendix~\ref{app:data}.

\textbf{Baselines.} We compare the proposed \method{} with a comprehensive set of  baselines on the above datasets under the same cross-site OOD settings, spanning four categories: (1) general graph neural networks (GNNs), including GCN~\cite{kipf2016gcn}, GAT~\cite{petar2018gat}, and GIN~\cite{xu2018gin}; (2) general OOD methods, such as IRM~\cite{arjovsky2019irm} and CORAL~\cite{sun2016coral}; (3) graph OOD methods, including GSAT~\cite{miao2022gsat}, DisC~\cite{fan2022disc}, CEPG~\cite{wang2026cepg}, and DiSCO~\cite{sun2026disco}; and (4) brain networks, including STAGIN~\cite{kim2021stagin}, BrainNetTF~\cite{kan2022brainnttf}, XG-GNN~\cite{qiu2024towards}, AGMGC~\cite{noman2025agmgc}, FC-HGNN~\cite{gu2025fchgnn}, LHDFormer~\cite{xue2025lhdformer}, NeuroH-TGL~\cite{li2025neurohtgl}, BrainOOD~\cite{xu2025brainood}, and DeCI~\cite{yu2026deci}. More details of baselines are provided in Appendix~\ref{app:baseline}.

\input{table/OOD}
\vspace{-0.1cm}
\subsection{Performance Comparison}

\begin{wrapfigure}{r}{0.55\textwidth}
    \vspace{-1.15cm}
    \centering
    \begin{subfigure}{0.48\linewidth}
        \centering
        \includegraphics[width=\linewidth]{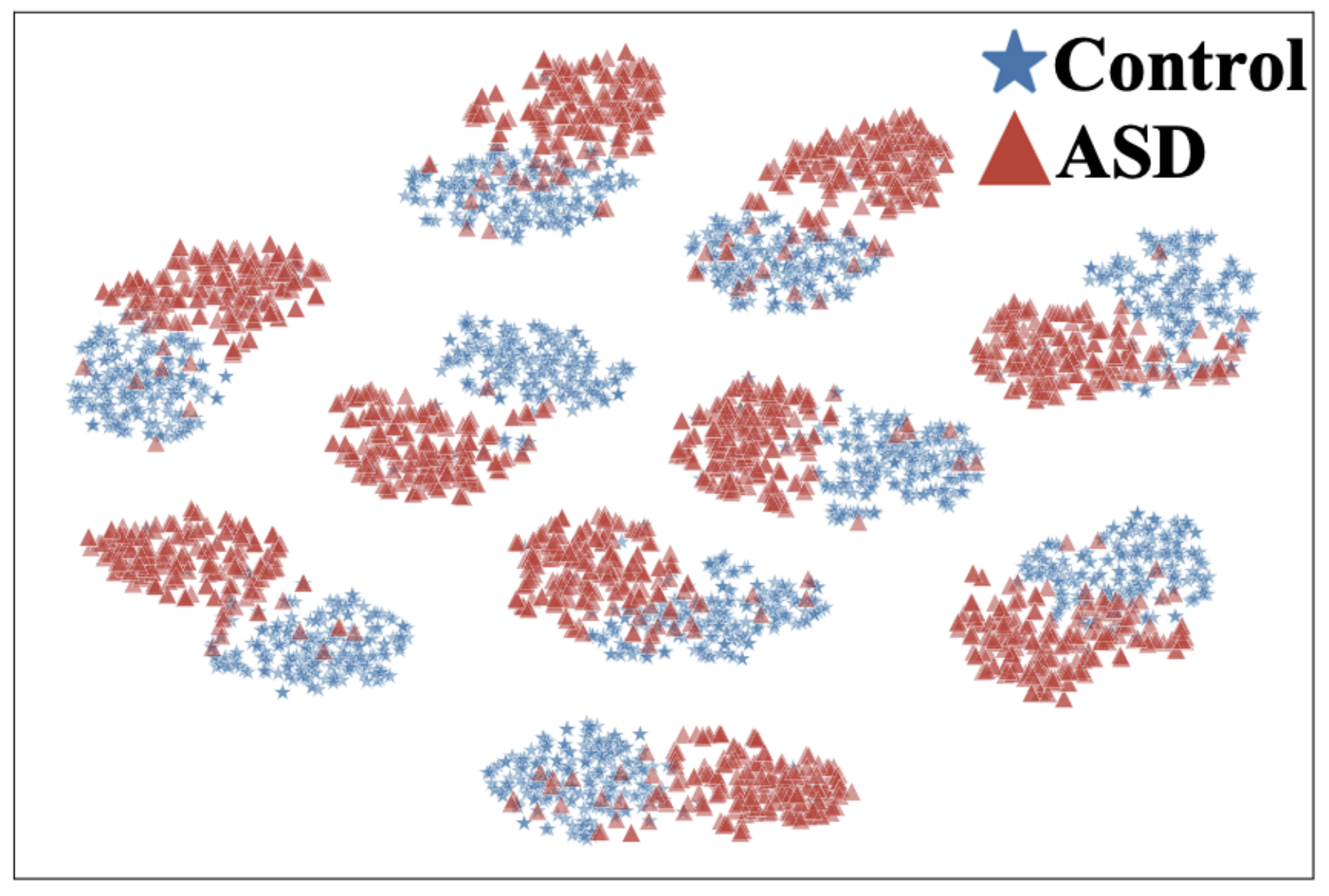}
        \caption{\method{}}
    \end{subfigure}
    \hfill
    \begin{subfigure}{0.48\linewidth}
        \centering
        \includegraphics[width=\linewidth]{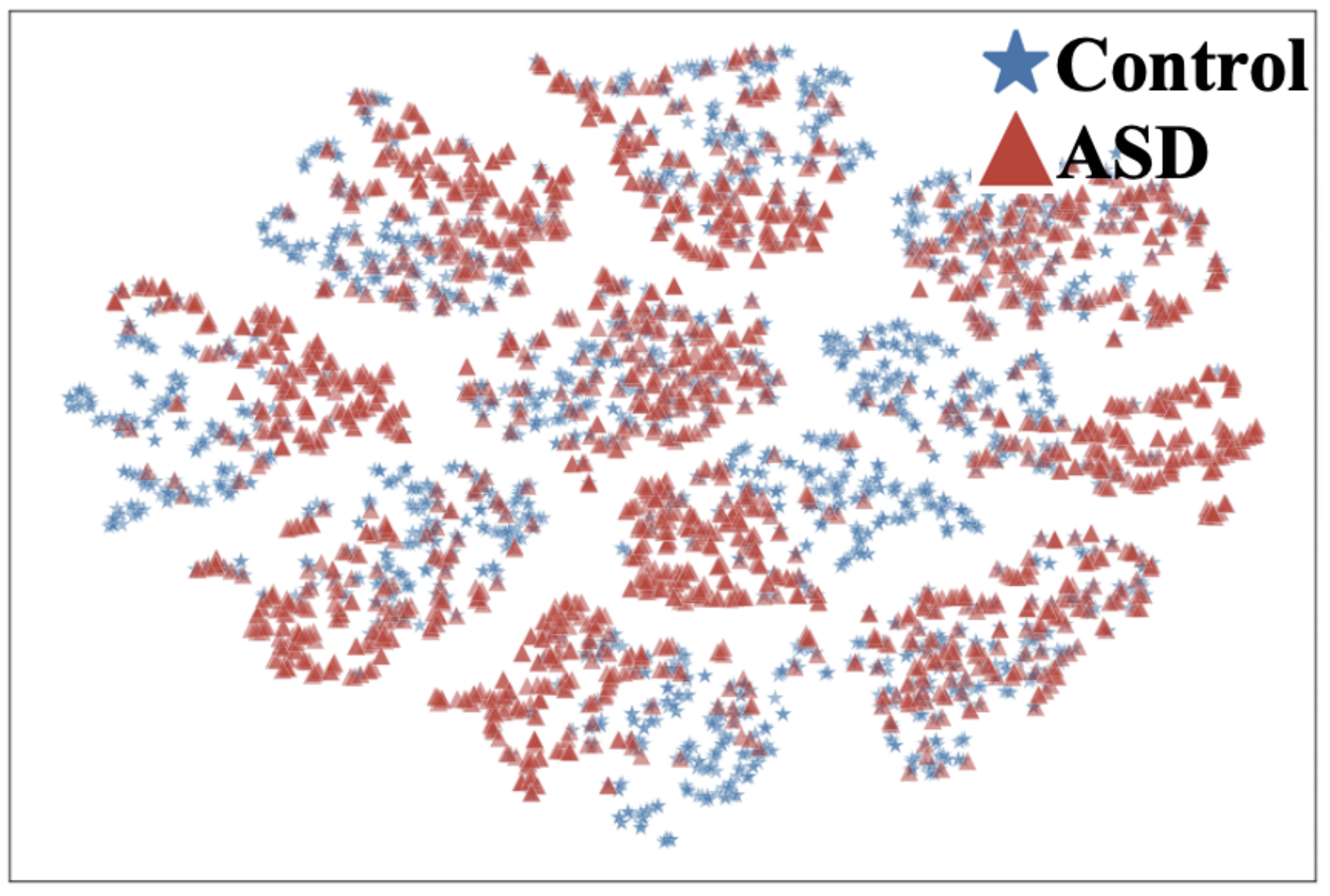}
        \caption{BrainOOD}
    \end{subfigure}
    \vspace{-0.1cm}
    \caption{t-SNE visualizations on ABIDE.}
    \label{fig:t-sne}
    \vspace{-0.5cm}
\end{wrapfigure}

We compare \method{} with all baseline methods under the LOSO protocol on the ABIDE, REST-meta-MDD, SRPBS, and ABCD datasets in Table~\ref{tab:main_results_ood}. We observe that: (1) Brain network methods achieve the best baseline results on all datasets and metrics, highlighting the importance of specialized FC modeling for cross-site neuroimaging prediction. (2) General OOD and graph OOD methods do not consistently improve performance across datasets, suggesting that generic robustness objectives alone may not adequately address site-conditioned variations in brain connectivity. (3) \method{} consistently outperforms all baselines across datasets and metrics. These improvements reflect three complementary factors: (i) site-aware confounder decoupling adjusts site-dependent covariate effects, while population scaffold extraction retains reproducible label-associated connections; (ii) transient pathway profiling summarizes scaffold-level dynamic FC into compact, fixed-dimensional descriptors, providing temporal features that complement static connectivity; and (iii) prior-guided subject-adaptive gating combines population-level scaffold priors with individual pathway features and context to emphasize informative pathways during message passing and readout, adapting their contributions to each subject's prediction. Additionally, we present t-SNE visualizations on the ABIDE dataset under cross-site OOD settings for a strong baseline, BrainOOD~\cite{xu2025brainood}, and \method{}, as shown in Fig.~\ref{fig:t-sne}. The embeddings produced by \method{} exhibit clearer separation between Control and ASD samples, indicating that it learns more discriminative representations under cross-site OOD settings. 

To further assess whether the performance gains depend on a specific parcellation, we evaluate an additional ABIDE (CC200) setting by replacing the AAL atlas with the CC200 functional parcellation~\cite{cc200}. CC200 comprises 200 functionally defined ROIs, resulting in different graph sizes, node definitions, and edge topologies compared to the anatomically defined AAL atlas. As shown in Table~\ref{tab:main_results_ood}, \method{} still outperforms baseline methods under this alternative parcellation, indicating that its gains are not solely dependent on the AAL atlas.

\subsection{Ablation Study}

To examine the contribution of each component in \method{}, we conduct ablation studies on four variants: (1) \method{} w/o CD, where site-aware deconfounding is replaced with global deconfounding; (2) \method{} w/o TP, where transient pathway profiling is replaced by static FC features; (3) \method{} w/o LG, where line-graph message passing is removed and scaffold nodes are independently encoded before gated readout; and (4) \method{} w/o SG, where the gating module is removed and all scaffold pathways receive uniform weights during message passing and readout. As shown in Fig.~\ref{fig:sensitivity}(a, b), we observe that: (1) Replacing site-aware deconfounding with global deconfounding reduces performance, indicating the importance of accounting for site-dependent confounder effects. Removing TP also degrades performance, suggesting that scaffold-level transient descriptors provide complementary dynamic information beyond static FC. (2) Removing LG or SG results in performance drops, highlighting the importance of interactions among scaffold pathways and subject-adaptive pathway selection for cross-site prediction. More results are provided in Appendix~\ref{app:ablation}.

\begin{figure}[t]
\centering
\begin{subfigure}[b]{0.23\linewidth}
    \centering
    \includegraphics[width=\linewidth]{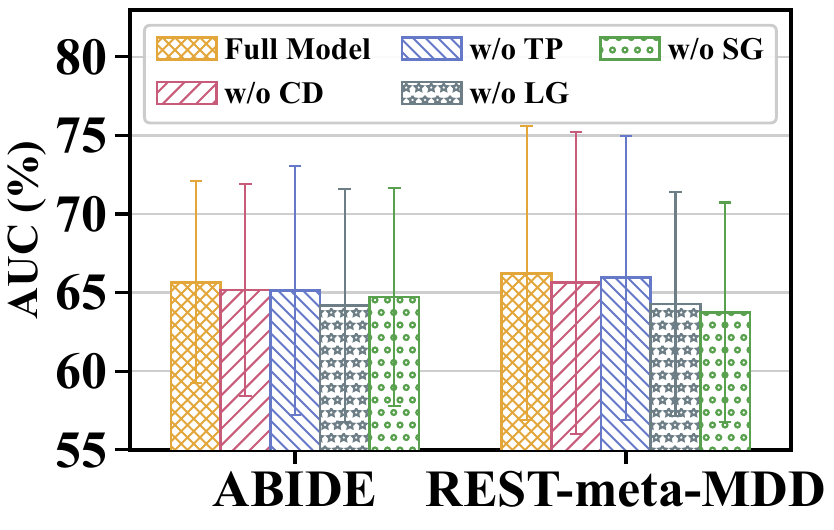}
    \caption{AUC}
\end{subfigure}
\hfill
\begin{subfigure}[b]{0.23\linewidth}
    \centering
    \includegraphics[width=\linewidth]{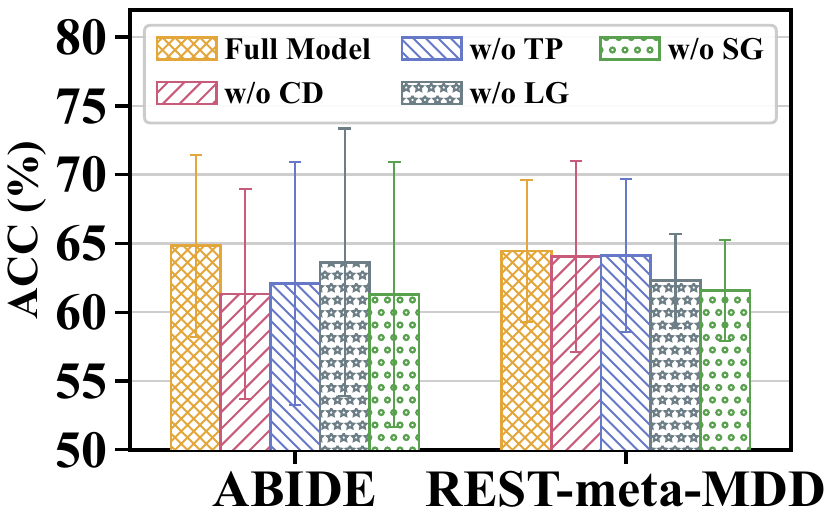}
    \caption{ACC}
\end{subfigure}
\hfill
\begin{subfigure}[b]{0.23\linewidth}
    \centering\includegraphics[height=0.623\linewidth]{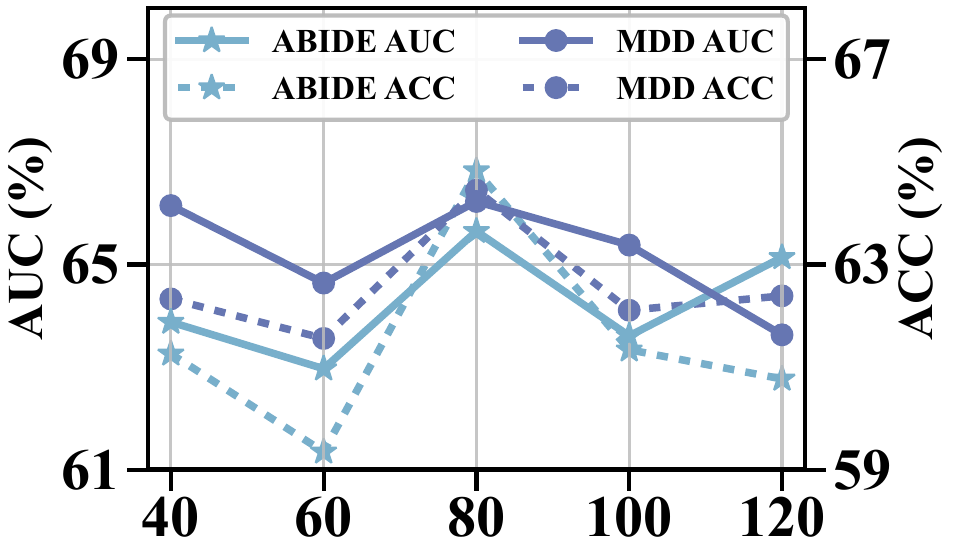}
    \caption{Gate Budget $K$}
\end{subfigure}
\hfill
\begin{subfigure}[b]{0.23\linewidth}
    \centering
    \includegraphics[width=\linewidth]{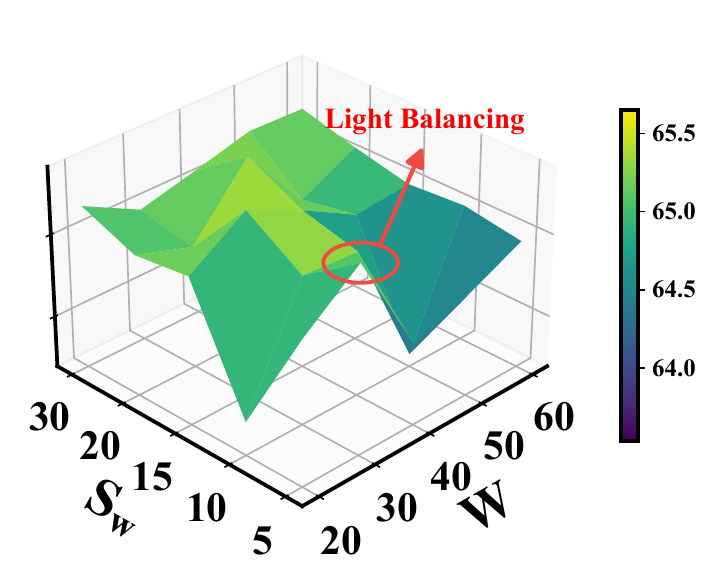}
    \caption{Windowing ($W$, $S_w$)}
\end{subfigure}
\vspace{-0.2cm}
\caption{Ablation results (a,b) and gate-budget sensitivity (c) on ABIDE and REST-meta-MDD, and windowing sensitivity (d) on ABIDE.}
\vspace{-0.5cm}
\label{fig:sensitivity}
\end{figure}

\subsection{Interpretability Analysis}
\begin{wrapfigure}{r}{0.38\textwidth}
     \vspace{-1.4cm}
     \centering\includegraphics[width=\linewidth]{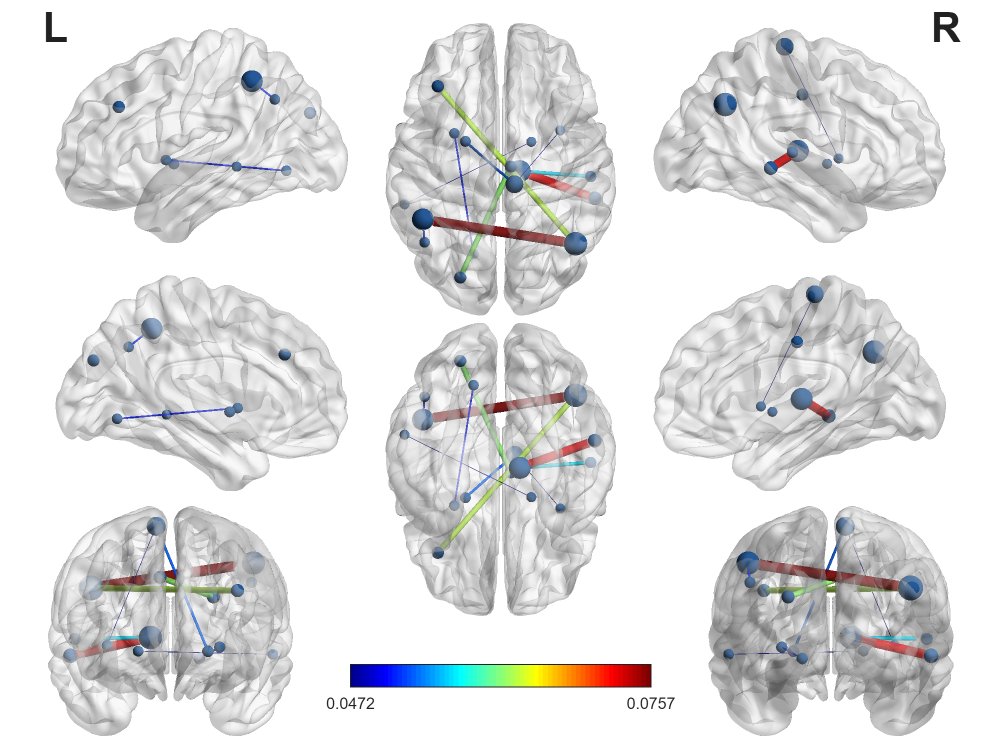}
     \vspace{-0.6cm}
     \caption{Top 10 positive scaffold connections on ABIDE.}
     \label{fig:edge}
     \vspace{-0.35cm}
\end{wrapfigure}

To assess the neurobiological relevance of the scaffold connections, Fig.~\ref{fig:edge} visualizes the top ten positive scaffold connections on ABIDE under cross-site OOD evaluation. As these connections are selected from the population scaffold, they are label-associated edges reproducible across training sites. The resulting pattern is sparse and spans cortico-subcortical circuits, including parietal association regions, thalamocortical pathways, cingulo-insular nodes, and basal ganglia structures. These regions and circuits are broadly consistent with prior rs-fMRI findings in ASD, which implicate default-mode/fronto-parietal, thalamocortical, salience, and striatal/basal-ganglia connectivity~\cite{hull2017resting,lee2016abnormalities,di2011aberrant}. Overall, these qualitative results suggest that the scaffold captures neurobiologically interpretable and cross-site reproducible connectivity patterns. A more detailed analysis of these positive connections is provided in Appendix~\ref{sec:interpre}.

\subsection{Sensitivity Study}

To assess how pathway weighting and temporal resolution affect cross-site prediction, we examine the sensitivity of \method{} to the gate budget $K$, window length $W$, and stride $S_w$. As shown in Fig.~\ref{fig:sensitivity}(c), \method{} achieves the highest AUC and ACC at $K=80$ on both ABIDE and REST-meta-MDD, with lower performance at the other tested budgets. This pattern suggests a trade-off in pathway weighting: a small budget may overly suppress informative pathways, whereas a large budget may increase the contributions of less informative pathways during message passing and readout. We evaluate different combinations of $W$ and $S_w$ on ABIDE. As shown in Fig.~\ref{fig:sensitivity}(d), \method{} achieves the best performance at $(W,S_w)=(30,5)$. Larger window lengths or strides generally degrade performance. This suggests that longer windows may smooth out brief connectivity changes, while larger strides sample them less densely, limiting the temporal information captured by the descriptors. These results support moderate window lengths and fine strides for balancing temporal detail with stable FC estimation. More results are provided in Appendix~\ref{app:sensitivity}.

%% file: table/OOD.tex
\begin{table*}[t]
    \centering
    \caption{Results on ABIDE, ABIDE (CC200), REST-meta-MDD, SRPBS, and ABCD (ADHD-task) over LOSO sites. \textbf{Bold} results indicate the best performance.}
    \vspace{-0.15cm}
    \label{tab:main_results_ood}
    \resizebox{\textwidth}{!}{
    \begin{tabular}{cc |cc |cc |cc |cc |cc}
    \toprule
    \multirow{2}{*}{Type} & \multirow{2}{*}{Method}
    & \multicolumn{2}{c|}{ABIDE}
    & \multicolumn{2}{c|}{ABIDE (CC200)}
    & \multicolumn{2}{c|}{REST-meta-MDD}
    & \multicolumn{2}{c|}{SRPBS}
    & \multicolumn{2}{c}{ABCD (ADHD-task)} \\
    &
    & AUC & ACC
    & AUC & ACC
    & AUC & ACC
    & AUC & ACC
    & AUC & ACC \\
    \midrule

    % ==================== General GNNs ====================
    \multirow{3}{*}{
        \rotatebox{90}{\fontsize{9}{12}\selectfont GNN}
    }
& GCN
    & 62.13$_{\pm 10.32}$ & 58.64$_{\pm 9.66}$
    & 50.43$_{\pm 8.86}$ & 50.41$_{\pm 8.55}$
    & 60.64$_{\pm 4.18}$ & 56.32$_{\pm 2.87}$
    & 73.81$_{\pm 19.46}$ & 70.23$_{\pm 19.17}$
    & 71.75$_{\pm 6.52}$ & 63.14$_{\pm 9.15}$ \\

    & GAT
    & 58.42$_{\pm 8.69}$ & 54.82$_{\pm 6.01}$
    & 57.90$_{\pm 9.37}$ & 56.71$_{\pm 7.74}$
    & 59.97$_{\pm 9.38}$ & 56.24$_{\pm 6.73}$
    & 75.22$_{\pm 15.63}$ & 69.38$_{\pm 14.42}$
    & 70.40$_{\pm 11.32}$ & 63.32$_{\pm 11.20}$ \\

    & GIN
    & 57.95$_{\pm 5.21}$ & 54.60$_{\pm 3.10}$
    & 52.61$_{\pm 8.83}$ & 49.86$_{\pm 10.23}$
    & 56.58$_{\pm 6.72}$ & 53.47$_{\pm 4.68}$
    & 72.33$_{\pm 21.67}$ & 68.74$_{\pm 18.09}$
    & 68.97$_{\pm 8.84}$ & 60.73$_{\pm 7.03}$ \\

    \midrule

    % ==================== General OOD ====================
    \multirow{2}{*}{
        \rotatebox{90}{\fontsize{9}{12}\selectfont OOD}
    }
    & CORAL
    & 55.69$_{\pm 6.12}$ & 52.75$_{\pm 5.49}$
    & 53.96$_{\pm 10.69}$ & 52.08$_{\pm 7.42}$
    & 60.42$_{\pm 8.97}$ & 57.19$_{\pm 7.83}$
    & 72.41$_{\pm 15.52}$ & 68.43$_{\pm 12.68}$
    & 69.48$_{\pm 7.70}$ & 62.55$_{\pm 10.46}$ \\

    & IRM
    & 58.77$_{\pm 6.12}$ & 55.50$_{\pm 3.85}$
    & 51.98$_{\pm 10.64}$ & 52.86$_{\pm 6.99}$
    & 58.63$_{\pm 4.38}$ & 54.22$_{\pm 3.41}$
    & 69.48$_{\pm 21.87}$ & 67.82$_{\pm 16.78}$
    & 70.40$_{\pm 6.60}$ & 59.06$_{\pm 7.57}$ \\

    \midrule

    % ==================== Graph OOD ====================
    \multirow{4}{*}{
        \rotatebox{90}{
            \shortstack{\fontsize{9}{12}\selectfont Graph\\OOD}
        }
    }
& GSAT
    & 58.69$_{\pm 7.62}$ & 56.12$_{\pm 7.98}$
    & 51.02$_{\pm 7.05}$ & 51.07$_{\pm 6.27}$
    & 57.38$_{\pm 5.82}$ & 54.83$_{\pm 4.47}$
    & 71.32$_{\pm 18.96}$ & 67.68$_{\pm 15.83}$
    & 71.27$_{\pm 8.73}$ & 61.16$_{\pm 7.48}$ \\

    & DisC
    & 56.67$_{\pm 7.85}$ & 56.14$_{\pm 8.68}$
    & 53.06$_{\pm 6.48}$ & 52.77$_{\pm 6.41}$
    & 56.21$_{\pm 4.42}$ & 53.28$_{\pm 3.88}$
    & 73.72$_{\pm 19.78}$ & 69.97$_{\pm 16.93}$
    & 69.59$_{\pm 7.70}$ & 58.83$_{\pm 9.45}$ \\

    & CEPG
    & 51.04$_{\pm 8.99}$ & 50.03$_{\pm 6.12}$
    & 45.03$_{\pm 9.74}$ & 47.47$_{\pm 8.93}$
    & 53.98$_{\pm 5.29}$ & 52.96$_{\pm 4.46}$
    & 63.83$_{\pm 8.68}$ & 59.12$_{\pm 6.27}$
    & 63.11$_{\pm 11.06}$ & 57.07$_{\pm 9.59}$ \\

    & DiSCO
    & 56.78$_{\pm 6.72}$ & 55.47$_{\pm 4.83}$
    & 51.58$_{\pm 8.45}$ & 51.37$_{\pm 6.21}$
    & 57.96$_{\pm 3.73}$ & 54.08$_{\pm 2.84}$
    & 74.18$_{\pm 18.18}$ & 67.64$_{\pm 15.46}$
    & 67.51$_{\pm 9.83}$ & 61.98$_{\pm 7.44}$ \\

    \midrule

    % ==================== Brain Networks ====================
    \multirow{9}{*}{
        \rotatebox{90}{
            \shortstack{\fontsize{9}{12}\selectfont Brain\\Networks}
        }
    }

    % 2021
    & STAGIN
    & 56.94$_{\pm 7.85}$ & 54.27$_{\pm 6.82}$
    & 54.73$_{\pm 9.41}$ & 54.16$_{\pm 9.65}$
    & 57.41$_{\pm 5.12}$ & 53.85$_{\pm 4.60}$
    & 73.15$_{\pm 18.55}$ & 68.42$_{\pm 17.35}$
    & 67.93$_{\pm 9.05}$ & 59.41$_{\pm 8.75}$ \\

    % 2022
    & BrainNetTF
    & 62.85$_{\pm 6.20}$ & 60.90$_{\pm 7.06}$
    & 57.04$_{\pm 11.21}$ & 52.55$_{\pm 8.08}$
    & 61.50$_{\pm 7.40}$ & 58.20$_{\pm 5.60}$
    & 61.56$_{\pm 17.03}$ & 60.41$_{\pm 10.80}$
    & 68.28$_{\pm 4.54}$ & 56.91$_{\pm 10.43}$ \\

    % 2024
    & XG-GNN
    & 60.41$_{\pm 7.73}$ & 54.12$_{\pm 4.09}$
    & 50.64$_{\pm 12.57}$ & 52.71$_{\pm 8.78}$
    & 60.08$_{\pm 7.23}$ & 56.24$_{\pm 6.49}$
    & 81.80$_{\pm 15.21}$ & 74.69$_{\pm 17.29}$
    & 71.20$_{\pm 5.46}$ & 55.66$_{\pm 9.78}$ \\

    % 2025
    & AGMGC
    & 58.68$_{\pm 8.68}$ & 51.49$_{\pm 5.42}$
    & 55.67$_{\pm 7.72}$ & 50.85$_{\pm 5.23}$
    & 59.06$_{\pm 5.07}$ & 54.74$_{\pm 4.12}$
    & 83.52$_{\pm 15.32}$ & 75.86$_{\pm 16.87}$
    & 70.86$_{\pm 8.67}$ & 63.02$_{\pm 11.29}$ \\

    & FC-HGNN
    & 55.26$_{\pm 7.90}$ & 52.23$_{\pm 4.58}$
    & 55.04$_{\pm 9.71}$ & 53.21$_{\pm 5.45}$
    & 58.69$_{\pm 6.69}$ & 56.78$_{\pm 5.71}$
    & 57.30$_{\pm 10.60}$ & 53.80$_{\pm 6.40}$
    & 71.08$_{\pm 5.74}$ & 59.10$_{\pm 12.34}$ \\

    & LHDFormer
    & 57.63$_{\pm 7.61}$ & 55.39$_{\pm 6.74}$
    & 56.45$_{\pm 9.11}$ & 55.27$_{\pm 9.48}$
    & 59.84$_{\pm 5.88}$ & 55.43$_{\pm 4.95}$
    & 75.63$_{\pm 17.20}$ & 69.87$_{\pm 16.05}$
    & 69.52$_{\pm 8.44}$ & 61.28$_{\pm 9.15}$ \\

    & NeuroH-TGL
    & 60.51$_{\pm 7.22}$ & 58.14$_{\pm 6.25}$
    & 59.38$_{\pm 8.71}$ & 57.96$_{\pm 8.94}$
    & 62.37$_{\pm 6.15}$ & 58.02$_{\pm 5.42}$
    & 85.12$_{\pm 12.45}$ & 78.34$_{\pm 13.80}$
    & 72.18$_{\pm 6.95}$ & 64.15$_{\pm 8.22}$ \\
    
    & BrainOOD
    & 63.82$_{\pm 6.73}$ & 59.66$_{\pm 9.70}$
    & 62.51$_{\pm 4.61}$ & 58.90$_{\pm 4.09}$
    & 63.22$_{\pm 6.29}$ & 58.45$_{\pm 5.20}$
    & 90.00$_{\pm 7.33}$ & 77.70$_{\pm 15.23}$
    & 72.77$_{\pm 8.07}$ & 64.36$_{\pm 7.24}$ \\

    % 2026
    & DeCI
    & 56.02$_{\pm 8.91}$ & 60.78$_{\pm 2.89}$
    & 58.55$_{\pm 3.92}$ & 61.93$_{\pm 3.72}$
    & 54.38$_{\pm 2.95}$ & 56.18$_{\pm 2.89}$
    & 90.57$_{\pm 12.30}$ & 85.21$_{\pm 13.68}$
    & 55.03$_{\pm 6.79}$ & 62.81$_{\pm 5.46}$ \\

    \midrule

    % ==================== Ours ====================
    & \method{}
    & \textbf{65.65$_{\pm 6.42}$} & \textbf{64.83$_{\pm 6.62}$}
    & \textbf{64.61$_{\pm 7.12}$} & \textbf{62.18$_{\pm 8.35}$}
    & \textbf{66.23$_{\pm 9.35}$} & \textbf{64.45$_{\pm 5.15}$}
    & \textbf{92.35$_{\pm 7.04}$} & \textbf{87.31$_{\pm 10.52}$}
    & \textbf{75.25$_{\pm 7.63}$} & \textbf{66.11$_{\pm 8.65}$} \\

    \bottomrule
    \end{tabular}
    }
    \vspace{-0.6cm}
\end{table*}

%% file: main/6_conclusion.tex
\vspace{-0.2cm}
\section{Conclusion}
\vspace{-0.1cm}

In this paper, we propose \method{}, a unified framework for cross-site OOD generalization in brain network analysis. It combines site-aware deconfounding with reliability-weighted aggregation for unseen-site correction and uses a population scaffold to organize temporal pathway representations and guide subject-adaptive graph learning. Leave-one-site-out experiments on four multi-site datasets demonstrate consistent improvements in unseen-site prediction, with gains maintained under an alternative parcellation. Future work will explore acquisition-aware temporal descriptors and extend deconfounding to settings with incomplete covariate information.

%% file: main/7_appendix.tex
%\appendix
%\clearpage

\section{Notation Summary}

\input{table/notations}

As shown in the Table~\ref{tab:notation}, we  summarize the key notations of this paper.

\section{Proof of Proposition~\ref{pro1}}

\textbf{Proposition~\ref{pro1} (Residual Bias of Pooled Deconfounding).} \textit{
Assume FC follows the site-specific additive model
\begin{equation}
    \bm{r}_i
    =
    \bm{b}_e^{\star}
    +
    \bm{\Gamma}_e^{\star\top}\tilde{\bm{q}}_i
    +
    \bm{s}_i^{\star}
    +
    \bm{\epsilon}_i,
\end{equation}
where $\bm{b}_e^{\star}\in\mathbb{R}^{M}$ and $\bm{\Gamma}_e^{\star}\in\mathbb{R}^{d\times M}$ are the site-specific intercept and covariate coefficients, $\bm{s}_i^{\star}\in\mathbb{R}^{M}$ is the confounder-free FC component with $\mathbb E[\bm s_i^{\star}\mid e,\tilde{\bm q}_i,Y_i]=\mathbb E[\bm s_i^{\star}\mid e,Y_i]$, and $\mathbb{E}[\bm{\epsilon}_i\mid e,\tilde{\bm{q}}_i,\bm{s}_i^{\star},Y_i]=\bm{0}$. Let $\mathrm D_e[\bm v]=\mathbb E[\bm v\mid e,Y_i=1]-\mathbb E[\bm v\mid e,Y_i=0]$ denote the diagnostic contrast at site $e$. Then, for any fixed global coefficients $(\bm{b}_{\mathrm{glob}},\bm{\Gamma}_{\mathrm{glob}})$,
\begin{equation}
    \mathrm D_e\big[\hat{\bm r}_i^{\mathrm{glob}}\big]
    =
    \mathrm D_e[\bm s_i^{\star}]
    +
    (\bm{\Gamma}_e^{\star}-\bm{\Gamma}_{\mathrm{glob}})^{\top}
    \mathrm D_e[\tilde{\bm q}_i],
\end{equation}
whereas the site-wise residual $\bm r_i-\bm b_e^{\star}-\bm\Gamma_e^{\star\top}\tilde{\bm q}_i$ satisfies $\mathrm D_e[\cdot]=\mathrm D_e[\bm s_i^{\star}]$.
}

\textit{\textbf{Proof.}}

Throughout, all random vectors have finite first moments, so that the conditional expectations below exist, and both diagnostic groups have positive probability at site $e$, so that $\mathrm D_e[\cdot]$ is well defined.

Since $\mathrm D_e[\cdot]$ is a difference of two conditional expectations, it inherits their linearity. For random vectors $\bm u,\bm v\in\mathbb R^{M}$ with finite first moments,
\begin{equation}
    \mathrm D_e[\bm u+\bm v]=\mathrm D_e[\bm u]+\mathrm D_e[\bm v].
    \label{eq:app_lemma_linear}
\end{equation}
A fixed vector $\bm c\in\mathbb R^{M}$ has the same conditional expectation in both groups and therefore cancels in the difference,
\begin{equation}
    \mathrm D_e[\bm c]=\bm 0,
    \label{eq:app_lemma_const}
\end{equation}
and for a fixed matrix $\bm A\in\mathbb R^{d\times M}$, linearity of conditional expectation gives
\begin{equation}
    \mathrm D_e[\bm A^{\top}\tilde{\bm q}_i]=\bm A^{\top}\mathrm D_e[\tilde{\bm q}_i].
    \label{eq:app_lemma_affine}
\end{equation}
The noise term also has zero contrast. By the tower property of conditional expectation and the assumption $\mathbb E[\bm\epsilon_i\mid e,\tilde{\bm q}_i,\bm s_i^{\star},Y_i]=\bm 0$, for each $y\in\{0,1\}$,
\begin{equation}
    \mathbb E[\bm\epsilon_i\mid e,Y_i=y]
    =
    \mathbb E\big[\,\mathbb E[\bm\epsilon_i\mid e,\tilde{\bm q}_i,\bm s_i^{\star},Y_i]\,\big|\,e,Y_i=y\big]
    =
    \bm 0,
\end{equation}
and hence
\begin{equation}
    \mathrm D_e[\bm\epsilon_i]=\bm 0 .
    \label{eq:app_lemma_noise}
\end{equation}

We now turn to the pooled residual. Fix a site $e$ and global coefficients $(\bm b_{\mathrm{glob}},\bm\Gamma_{\mathrm{glob}})$. Substituting the generative model into $\hat{\bm r}_i^{\mathrm{glob}}=\bm r_i-\bm b_{\mathrm{glob}}-\bm\Gamma_{\mathrm{glob}}^{\top}\tilde{\bm q}_i$ and collecting terms gives, for every subject $i$ at site $e$,
\begin{equation}
\label{eq:app_glob_decomp}
    \hat{\bm r}_i^{\mathrm{glob}}
    =
    \bm s_i^{\star}
    +\big(\bm b_e^{\star}-\bm b_{\mathrm{glob}}\big)
    +\big(\bm\Gamma_e^{\star}-\bm\Gamma_{\mathrm{glob}}\big)^{\top}\tilde{\bm q}_i
    +\bm\epsilon_i .
\end{equation}
The four terms are the confounder-free component, an intercept mismatch that is constant across subjects at site $e$, a covariate-dependent slope mismatch, and noise. Applying $\mathrm D_e[\cdot]$ to both sides and expanding by Eq.~\eqref{eq:app_lemma_linear},
\begin{equation}
\label{eq:app_glob_contrast_expand}
    \mathrm D_e\big[\hat{\bm r}_i^{\mathrm{glob}}\big]
    =
    \mathrm D_e[\bm s_i^{\star}]
    +\mathrm D_e\big[\bm b_e^{\star}-\bm b_{\mathrm{glob}}\big]
    +\mathrm D_e\big[(\bm\Gamma_e^{\star}-\bm\Gamma_{\mathrm{glob}})^{\top}\tilde{\bm q}_i\big]
    +\mathrm D_e[\bm\epsilon_i] .
\end{equation}
The intercept term vanishes by Eq.~\eqref{eq:app_lemma_const}, since $\bm b_e^{\star}-\bm b_{\mathrm{glob}}$ is fixed given $e$; the slope term equals $(\bm\Gamma_e^{\star}-\bm\Gamma_{\mathrm{glob}})^{\top}\mathrm D_e[\tilde{\bm q}_i]$ by Eq.~\eqref{eq:app_lemma_affine} with $\bm A=\bm\Gamma_e^{\star}-\bm\Gamma_{\mathrm{glob}}$; and the noise term vanishes by Eq.~\eqref{eq:app_lemma_noise}. Substituting back into Eq.~\eqref{eq:app_glob_contrast_expand} yields
\begin{equation}
    \mathrm D_e\big[\hat{\bm r}_i^{\mathrm{glob}}\big]
    =
    \mathrm D_e[\bm s_i^{\star}]
    +(\bm\Gamma_e^{\star}-\bm\Gamma_{\mathrm{glob}})^{\top}\mathrm D_e[\tilde{\bm q}_i],
\end{equation}
which is the first claim of the proposition.

The site-wise residual is handled in the same way. Substituting the generative model gives
\begin{equation}
\label{eq:app_site_decomp}
    \bm r_i-\bm b_e^{\star}-\bm\Gamma_e^{\star\top}\tilde{\bm q}_i
    =
    \bm s_i^{\star}+\bm\epsilon_i ,
\end{equation}
so that, by Eqs.~\eqref{eq:app_lemma_linear} and~\eqref{eq:app_lemma_noise},
\begin{equation}
    \mathrm D_e\big[\bm r_i-\bm b_e^{\star}-\bm\Gamma_e^{\star\top}\tilde{\bm q}_i\big]
    =
    \mathrm D_e[\bm s_i^{\star}]+\mathrm D_e[\bm\epsilon_i]
    =
    \mathrm D_e[\bm s_i^{\star}],
\end{equation}
which is the second claim.

The decomposition in Eq.~\eqref{eq:app_glob_decomp} also describes the bias at the subject level. Conditioning it on $(e,\tilde{\bm q}_i,\bm s_i^{\star})$ rather than on $(e,Y_i)$, and using $\mathbb E[\bm\epsilon_i\mid e,\tilde{\bm q}_i,\bm s_i^{\star}]=\bm 0$, which again follows from the stated assumption by the tower property, gives
\begin{equation}
    \mathbb E\big[\hat{\bm r}_i^{\mathrm{glob}}\mid e,\tilde{\bm q}_i,\bm s_i^{\star}\big]
    =
    \bm s_i^{\star}
    +(\bm b_e^{\star}-\bm b_{\mathrm{glob}})
    +(\bm\Gamma_e^{\star}-\bm\Gamma_{\mathrm{glob}})^{\top}\tilde{\bm q}_i .
\end{equation}
Both mismatch terms therefore bias the pooled residual relative to $\bm s_i^{\star}$ for each subject, but only the slope mismatch survives in the diagnostic contrast, where it enters through the between-group covariate difference $\mathrm D_e[\tilde{\bm q}_i]$. The same conditioning applied to Eq.~\eqref{eq:app_site_decomp} gives $\mathbb E[\bm r_i-\bm b_e^{\star}-\bm\Gamma_e^{\star\top}\tilde{\bm q}_i\mid e,\tilde{\bm q}_i,\bm s_i^{\star}]=\bm s_i^{\star}$, so the site-wise residual is unbiased at the subject level as well.

\section{Proof of Proposition~\ref{prop:shortcut_suppression}}
\label{app:prop2}

\textbf{Proposition~\ref{prop:shortcut_suppression} (Margin Stability of the Site-aware Scaffold)}.
\textit{
For edge $j\in\mathcal I$, let $r^{\mathrm{site}}_{i,j}=r_{i,j}-\hat b_{e,j}-\tilde{\bm q}_i^{\top}\hat{\bm\gamma}_{e,j}$ and $r^{\star}_{i,j}=r_{i,j}-b^{\star}_{e,j}-\tilde{\bm q}_i^{\top}\bm\gamma^{\star}_{e,j}$ denote the estimated and oracle site-wise residuals. Let $\mathcal S^{\star}$ be the scaffold obtained by applying the same selection rule, thresholds, and bootstrap draws to the oracle residuals, with oracle contrasts $d^{\star}_{e,j}$, $d^{\star}_{\mathrm{com},j}$, and $d^{\star(b)}_{\mathrm{com},j}$. Define the estimation perturbation and the oracle margin of edge $j$ as
\begin{gather}
    \Delta_j
    =
    \max_{e\in\mathcal E_{\mathrm{tr}}}
    \sum_{y\in\{0,1\}}
    \sup_{i\in\mathcal D_e^{(y)}}
    \big|r^{\mathrm{site}}_{i,j}-r^{\star}_{i,j}\big|,
    \\
    m^{\star}_j
    =
    \min\Big\{
    \min_{e\in\mathcal E_{\mathrm{tr}}}|d^{\star}_{e,j}|,\;
    \min_{1\le b\le B}|d^{\star(b)}_{\mathrm{com},j}|,\;
    |d^{\star}_{\mathrm{com},j}|,\;
    \big||d^{\star}_{\mathrm{com},j}|-\tau_E\big|
    \Big\}.
\end{gather}
If $\Delta_j<m^{\star}_j$, then $j\in\mathcal S$ if and only if $j\in\mathcal S^{\star}$.
}

\textit{\textbf{Proof.}}

Fix an edge $j\in\mathcal I$, the source data, and the fitted and oracle deconfounder parameters. The empirical and oracle scaffolds share the standardized covariates, the thresholds $(\tau_E,\eta_E,\zeta_E)$, and the $B$ site-bootstrap draws, so the argument is deterministic. We use the conventions of the main text: site-level contrasts are differences of Huber location estimates with a fixed threshold $\delta>0$, consensus contrasts are component-wise medians, nonunique minimizers of either estimator are resolved by their midpoint, $\operatorname{sgn}(0)=0$, and every source site contains both diagnostic groups.

We begin with the site-level contrasts. For a subject $i$ at site $e$, the empirical and oracle residuals differ by
\begin{equation}
    r^{\mathrm{site}}_{i,j}-r^{\star}_{i,j}
    =
    -\Big[(\hat b_{e,j}-b^{\star}_{e,j})+\tilde{\bm q}_i^{\top}(\hat{\bm\gamma}_{e,j}-\bm\gamma^{\star}_{e,j})\Big].
\end{equation}
For each class $y\in\{0,1\}$, define $a_{e,j}^{(y)}=\sup_{i\in\mathcal D_e^{(y)}}|r^{\mathrm{site}}_{i,j}-r^{\star}_{i,j}|$, so that $|r^{\mathrm{site}}_{i,j}-r^{\star}_{i,j}|\le a_{e,j}^{(y)}$ for every $i\in\mathcal D_e^{(y)}$, and write $\Delta_{e,j}=a_{e,j}^{(1)}+a_{e,j}^{(0)}$. By definition, $\Delta_j=\max_{e\in\mathcal E_{\mathrm{tr}}}\Delta_{e,j}$.

To pass from residuals to contrasts we need the Huber location estimator to be non-expansive under $\ell_\infty$ perturbations. For a nonempty finite sample $\bm x=(x_i)_{i\in A}$, let
\begin{equation}
    \mathcal M_\delta(\bm x)
    =
    \arg\min_{\mu\in\mathbb R}\sum_{i\in A}\mathcal H_\delta(x_i-\mu),
    \qquad
    T_\delta(\bm x)
    =
    \tfrac12\big(\min\mathcal M_\delta(\bm x)+\max\mathcal M_\delta(\bm x)\big).
\end{equation}
The objective is convex, continuous, and coercive, so $\mathcal M_\delta(\bm x)$ is a nonempty compact interval. Since the Huber loss is differentiable, its minimizers are exactly the zeros of the score sum
\begin{equation}
    \Psi_{\bm x}(\mu)=\sum_{i\in A}\psi_\delta(x_i-\mu),
    \qquad
    \psi_\delta(t)
    =
    \begin{cases}
        t, & |t|\le\delta,\\
        \delta\operatorname{sgn}(t), & |t|>\delta,
    \end{cases}
\end{equation}
which is continuous and non-increasing in $\mu$ with positive and negative limits as $\mu\to-\infty$ and $\mu\to+\infty$. Its zero set is therefore the interval with endpoints
\begin{equation}
    \min\mathcal M_\delta(\bm x)=\inf\{\mu:\Psi_{\bm x}(\mu)\le0\},
    \qquad
    \max\mathcal M_\delta(\bm x)=\sup\{\mu:\Psi_{\bm x}(\mu)\ge0\}.
\end{equation}
Two properties of $T_\delta$ follow from this characterization. First, $T_\delta$ is monotone: if $x_i\le z_i$ for all $i\in A$, then $\psi_\delta$ being non-decreasing gives $\Psi_{\bm x}(\mu)\le\Psi_{\bm z}(\mu)$ for all $\mu$, so both endpoints of the minimizer interval, and hence $T_\delta$, can only increase,
\begin{equation}
    T_\delta(\bm x)\le T_\delta(\bm z).
\end{equation}
Second, $T_\delta$ is translation equivariant: shifting every observation by a constant $c$ shifts the minimizer interval by $c$, so
\begin{equation}
    T_\delta(\bm x+c\bm 1)=T_\delta(\bm x)+c .
\end{equation}
Now let $a=\|\bm x-\bm x^{\star}\|_\infty$, so that $\bm x^{\star}-a\bm 1\le\bm x\le\bm x^{\star}+a\bm 1$ coordinate-wise. Applying monotonicity to both inequalities and translation equivariance to the bounds gives $T_\delta(\bm x^{\star})-a\le T_\delta(\bm x)\le T_\delta(\bm x^{\star})+a$, that is,
\begin{equation}
\label{eq:app_huber_nonexpansive}
    |T_\delta(\bm x)-T_\delta(\bm x^{\star})|\le\|\bm x-\bm x^{\star}\|_\infty .
\end{equation}
Let $\mu_{e,j}^{(y)}$ and $\mu_{e,j}^{\star(y)}$ denote the Huber locations of the empirical and oracle residuals in class $y$ at site $e$. Applying Eq.~\eqref{eq:app_huber_nonexpansive} to each class and using $d_{e,j}=\mu_{e,j}^{(1)}-\mu_{e,j}^{(0)}$ and $d^{\star}_{e,j}=\mu_{e,j}^{\star(1)}-\mu_{e,j}^{\star(0)}$,
\begin{equation}
\label{eq:app_site_contrast_bound}
    |d_{e,j}-d^{\star}_{e,j}|
    \le
    |\mu_{e,j}^{(1)}-\mu_{e,j}^{\star(1)}|+|\mu_{e,j}^{(0)}-\mu_{e,j}^{\star(0)}|
    \le
    a_{e,j}^{(1)}+a_{e,j}^{(0)}
    =
    \Delta_{e,j}
    \le
    \Delta_j
\end{equation}
for every $e\in\mathcal E_{\mathrm{tr}}$.

We next propagate this bound through the consensus and bootstrap contrasts, both of which are weighted medians of the site-level contrasts. For weights $w_e\ge0$ with $\sum_e w_e>0$, let
\begin{equation}
    \mathcal M_w(\bm x)=\arg\min_{u\in\mathbb R}\sum_e w_e|x_e-u|,
    \qquad
    \operatorname{Med}_w(\bm x)=\tfrac12\big(\min\mathcal M_w(\bm x)+\max\mathcal M_w(\bm x)\big),
\end{equation}
and let $F_{\bm x,w}(u)=\sum_e w_e\mathbb I[x_e\le u]/\sum_e w_e$. The endpoints of the weighted-median interval are $\min\mathcal M_w(\bm x)=\inf\{u:F_{\bm x,w}(u)\ge1/2\}$ and $\max\mathcal M_w(\bm x)=\inf\{u:F_{\bm x,w}(u)>1/2\}$. If $x_e\le z_e$ for all $e$, then $F_{\bm x,w}(u)\ge F_{\bm z,w}(u)$ for every $u$, so both endpoints and their midpoint are coordinate-wise monotone, and translating the observations translates both endpoints. The sandwich argument used for Eq.~\eqref{eq:app_huber_nonexpansive} therefore gives
\begin{equation}
\label{eq:app_median_nonexpansive}
    |\operatorname{Med}_w(\bm x)-\operatorname{Med}_w(\bm x^{\star})|
    \le
    \|\bm x-\bm x^{\star}\|_\infty .
\end{equation}
The consensus contrast is the equal-weight median, $d_{\mathrm{com},j}=\operatorname{Med}_{\bm 1}((d_{e,j})_{e})$, so Eqs.~\eqref{eq:app_site_contrast_bound} and~\eqref{eq:app_median_nonexpansive} give
\begin{equation}
\label{eq:app_com_bound}
    |d_{\mathrm{com},j}-d^{\star}_{\mathrm{com},j}|
    \le
    \max_{e\in\mathcal E_{\mathrm{tr}}}|d_{e,j}-d^{\star}_{e,j}|
    \le
    \Delta_j .
\end{equation}
For bootstrap replicate $b$, let $w_e^{(b)}$ be the number of times site $e$ is drawn, so that $\sum_e w_e^{(b)}=|\mathcal E_{\mathrm{tr}}|$; these multiplicities are the same in the empirical and oracle constructions and are unrelated to the deconfounder aggregation weights. The median of the resampled contrasts is the weighted median with these multiplicities, $d^{(b)}_{\mathrm{com},j}=\operatorname{Med}_{w^{(b)}}((d_{e,j})_{e})$ and $d^{\star(b)}_{\mathrm{com},j}=\operatorname{Med}_{w^{(b)}}((d^{\star}_{e,j})_{e})$, so the same argument gives
\begin{equation}
\label{eq:app_boot_bound}
    |d^{(b)}_{\mathrm{com},j}-d^{\star(b)}_{\mathrm{com},j}|\le\Delta_j,
    \qquad b=1,\dots,B .
\end{equation}

It remains to show that no quantity entering the selection rule crosses its threshold. Suppose $\Delta_j<m^{\star}_j$. Then $\Delta_j$ is strictly smaller than $|d^{\star}_{e,j}|$ for every $e$, than $|d^{\star}_{\mathrm{com},j}|$, and than $|d^{\star(b)}_{\mathrm{com},j}|$ for every $b$; in particular all of these oracle contrasts are nonzero. A real number within distance $\Delta_j$ of a nonzero $x$ with $|x|>\Delta_j$ is nonzero and has the sign of $x$, so Eqs.~\eqref{eq:app_site_contrast_bound}, \eqref{eq:app_com_bound}, and~\eqref{eq:app_boot_bound} yield
\begin{equation}
    \operatorname{sgn}(d_{e,j})=\operatorname{sgn}(d^{\star}_{e,j}),
    \quad
    \operatorname{sgn}(d_{\mathrm{com},j})=\operatorname{sgn}(d^{\star}_{\mathrm{com},j}),
    \quad
    \operatorname{sgn}(d^{(b)}_{\mathrm{com},j})=\operatorname{sgn}(d^{\star(b)}_{\mathrm{com},j})
\end{equation}
for all $e\in\mathcal E_{\mathrm{tr}}$ and $b=1,\dots,B$. Every indicator in the definitions of the sign consistency $\kappa_j$ and the bootstrap stability $\pi_j$ is therefore unchanged, and
\begin{equation}
    \kappa_j=\kappa_j^{\star},\qquad\pi_j=\pi_j^{\star}.
\end{equation}
For the magnitude criterion, the reverse triangle inequality and Eq.~\eqref{eq:app_com_bound} give
\begin{equation}
\big||d_{\mathrm{com},j}|-|d^{\star}_{\mathrm{com},j}|\big| \le |d_{\mathrm{com},j}-d^{\star}_{\mathrm{com},j}| \le \Delta_j < \big||d^{\star}_{\mathrm{com},j}|-\tau_E\big| .
\end{equation}
The right-hand side is positive, so $|d^{\star}_{\mathrm{com},j}|\neq\tau_E$. If $|d^{\star}_{\mathrm{com},j}|>\tau_E$, then $|d_{\mathrm{com},j}|\ge|d^{\star}_{\mathrm{com},j}|-\Delta_j>\tau_E$; if $|d^{\star}_{\mathrm{com},j}|<\tau_E$, then $|d_{\mathrm{com},j}|\le|d^{\star}_{\mathrm{com},j}|+\Delta_j<\tau_E$. Hence
\begin{equation}
    \mathbb I\big[|d_{\mathrm{com},j}|>\tau_E\big]
    =
    \mathbb I\big[|d^{\star}_{\mathrm{com},j}|>\tau_E\big].
\end{equation}
Since the empirical and oracle scaffolds apply the same thresholds to the three criteria,
\begin{equation}
\begin{aligned}
    \mathbf 1\{j\in\mathcal S\}
    &=
    \mathbb I\big[|d_{\mathrm{com},j}|>\tau_E,\;\kappa_j\ge\eta_E,\;\pi_j\ge\zeta_E\big]
    \\
    &=
    \mathbb I\big[|d^{\star}_{\mathrm{com},j}|>\tau_E,\;\kappa^{\star}_j\ge\eta_E,\;\pi^{\star}_j\ge\zeta_E\big]
    =
    \mathbf 1\{j\in\mathcal S^{\star}\},
\end{aligned}
\end{equation}
which completes the proof.

\input{table/dataset_description}

\section{Datasets}
\label{app:data}

\subsection{Dataset Descriptions}

Table~\ref{tab:original_data} summarizes the demographic and clinical
characteristics of the ABIDE, REST-meta-MDD, SRPBS, and ABCD datasets.
Detailed per-environment statistics are provided in
Tables~\ref{tab:abide_data}--\ref{tab:abcd_data}.

\textbf{ABIDE.}
The Autism Brain Imaging Data Exchange (ABIDE)~\cite{abide} is a multi-site resting-state fMRI consortium for autism research. Starting from the quality-controlled cohort, we retain 10 sites that (i) include both individuals with autism spectrum disorder (ASD) and typically developing controls (TD), and (ii) provide sufficient samples per class for site-level evaluation. The resulting benchmark contains 435 subjects (209 ASD / 226 TD). The retained sites exhibit substantial heterogeneity in acquisition settings and cohort composition, including sequence length ($T=140$--$209$ time points), site-level mean age (10.0--34.4 years), and sex distribution; SBL and TRINITY contain only male participants, whereas the other sites include both sexes. Table~\ref{tab:abide_data} summarizes the per-site statistics.

In addition to the standard AAL-116 setting, we evaluate ABIDE with the CC200 functional parcellation~\cite{cc200}. The two settings share the same 435 subjects, diagnostic labels, scan sessions, and sites, and use identical training, validation, and test partitions in every fold, so the CC200 experiment isolates the effect of parcellation rather than introducing an additional cohort.

\textbf{REST-meta-MDD.}
The REST-meta-MDD consortium~\cite{restmeatmdd} aggregates
resting-state fMRI data for major depressive disorder (MDD) across
multiple clinical sites in China. We retain the 10 largest sites with
adequate class representation, requiring each diagnostic class to contain
at least 15 subjects per site. The resulting cohort contains
1{,}596 subjects (751 healthy controls [HC] / 845 MDD), including
629 males and 967 females. Site-level mean ages range from 28.5 to
67.3 years.

The retained cohort also exhibits substantial site-size imbalance.
In particular, site S20 contains 533 subjects, approximately 33\% of
the retained cohort, whereas site S25 has a markedly older population
with a mean age of 67.3 years. These characteristics provide substantial
variation in population composition across evaluation environments.
Table~\ref{tab:mdd_data} summarizes the per-site statistics.

\textbf{SRPBS.}
The Strategic Research Program for Brain Sciences (SRPBS) dataset~\cite{srpbs} is a multi-site, multi-disorder neuroimaging resource released through the DecNef consortium in Japan. We focus on the MDD cohort and retain 10 site/batch environments after quality control and cohort filtering, yielding 968 subjects (484 HC / 484 MDD), including 437 males and 531 females. Participants span a broad adult age range (18--80 years), with similar mean ages in the MDD and HC groups ($43.0\pm13.5$ and $43.0\pm13.7$ years, respectively), and the retained environments differ in acquisition length, with $T\in\{107,143,240\}$ time points.

Four institutions (HUH, HKH, COI, and UTO) contribute two data batches each, corresponding to the distinct acquisition protocol identifiers provided in the SRPBS release. The two batches of an institution contain non-overlapping subjects and are denoted \texttt{Institution-batch1} and \texttt{Institution-batch2} in Table~\ref{tab:srpbs_data}. Because protocol differences within an institution induce shifts that are not captured by site identity alone, each site/batch unit is treated as an independent environment rather than being merged with the other batch from the same institution.

\textbf{ABCD.}
The Adolescent Brain Cognitive Development (ABCD)
Study\textsuperscript{\textregistered}~\cite{abcd} is a large-scale,
multi-site longitudinal cohort designed to characterize brain
development, cognition, and mental health from childhood into early
adulthood. After quality control and site filtering, we retain
10 imaging sites with sufficient per-class coverage for
attention-deficit/hyperactivity disorder (ADHD) versus HC
classification. The resulting cohort contains 425 subjects
(213 ADHD / 212 HC). All retained sessions contain $T=383$ time
points, and participants are narrowly concentrated around 9--10 years
of age (mean $9.47 \pm 0.50$ years). Therefore, cross-site
heterogeneity in this benchmark primarily reflects scanner/protocol
variation, sex-ratio differences, and site-specific class composition,
rather than broad age variation.
Table~\ref{tab:abcd_data} summarizes the per-site statistics.

Data used in the preparation of this article were obtained from the
ABCD Study\textsuperscript{\textregistered}, available through the
NIMH Data Archive (NDA). The ABCD Study is supported by the
National Institutes of Health and additional federal partners under
multiple award numbers, including U01DA041048, U01DA050989, and
U01DA051016, among others. ABCD consortium investigators designed
and implemented the study and/or provided data but did not necessarily
participate in the analysis or writing of this report. This manuscript
reflects the authors' views and does not necessarily represent those
of the NIH or ABCD consortium investigators.

\subsection{Dataset Preprocessing}

\textbf{ABIDE.}
For ABIDE, we use resting-state fMRI and phenotypic data from the Autism Brain Imaging Data Exchange~\cite{abide}, with preprocessed ROI time series obtained from the Preprocessed Connectomes Project (PCP). The standard setting uses the PCP-provided AAL-116 derivatives~\cite{aal}, and the ABIDE (CC200) setting uses the PCP-provided CC200 time series~\cite{cc200} produced by the same preprocessing pipeline. Subjects are retained only if they have valid diagnostic labels, site identifiers, the phenotypic variables required by the model, and usable ROI time series after quality control.

The two parcellation settings are matched exactly at the subject and
scan levels: both contain the same 435 subjects, the same diagnostic
labels and site identifiers, and the same scan sessions. Within every
evaluation fold, the AAL-116 and CC200 settings also use identical
training, validation, and test assignments. Thus, the atlas comparison
controls the cohort and evaluation partitions while changing the ROI
definition from 116 anatomical regions to 200 functionally defined
regions. Subjects are retained only if they have valid diagnostic
labels, site identifiers, phenotypic information required by the model,
and usable ROI time series after quality control.

\textbf{REST-meta-MDD.}
For REST-meta-MDD, we use consortium-provided derivatives that have
undergone standardized local preprocessing rather than processing the
raw imaging data ourselves~\cite{restmeatmdd}. We use AAL-116 ROI
time series~\cite{aal} together with diagnostic labels, site identifiers,
and the demographic variables required by the model. Subjects without
the required metadata or usable ROI time series are excluded.

\textbf{SRPBS.}
For SRPBS, we use multi-site resting-state fMRI data and associated
metadata released through the DecNef consortium~\cite{srpbs}.
Following the dataset-provided quality-control and preprocessing
procedures, the rs-fMRI data are parcellated using the AAL-116
atlas~\cite{aal}. The mean BOLD signal within each ROI is extracted
to form $
\bm{X}_i \in \mathbb{R}^{T_i \times 116}$. The site/batch identifiers defined above are retained throughout the
evaluation so that each of the 10 non-overlapping site/batch cohorts
constitutes a distinct environment.

\textbf{ABCD.}
For ABCD, we use data from the Adolescent Brain Cognitive Development
Study\textsuperscript{\textregistered}~\cite{abcd}, obtained through
the NIMH Data Archive (Release 4.0). We focus on the baseline
resting-state fMRI session
(\texttt{baseline\_year\_1\_arm\_1}).
Each subject's preprocessed rs-fMRI data are parcellated using the
AAL-116 atlas~\cite{aal}, and the mean BOLD signal within each ROI
is extracted to obtain
$
\bm{X}_i \in \mathbb{R}^{T_i \times 116}$.
Subjects are retained only when the required diagnostic labels,
site identifiers, demographic variables, and ROI time series are
available.

\section{Baselines}
\label{app:baseline}
\subsection{Baseline Description}
We compare \method{} with different representative baselines across four categories: general GNNs, general OOD methods, graph OOD methods, brain-network models.

\textbf{General GNNs.}
\begin{itemize}
    \item \textbf{GCN}~\cite{kipf2016gcn} is a spectral graph convolutional network derived from a first-order localized approximation of spectral graph convolutions, aggregating neighbor features via a symmetric normalized Laplacian.
    \item \textbf{GAT}~\cite{petar2018gat} introduces masked self-attention layers so that each node learns adaptive weights over its neighbors, alleviating the uniform-aggregation limitation of GCN.
    \item \textbf{GIN}~\cite{xu2018gin} adopts a sum aggregator together with an injective update function, matching the expressive power of the Weisfeiler--Lehman graph isomorphism test.
\end{itemize}

\textbf{General OOD methods.}
\begin{itemize}
    \item \textbf{IRM}~\cite{arjovsky2019irm} learns representations whose optimal classifier is invariant across training environments, penalizing environment-specific features for out-of-distribution generalization.
    \item \textbf{CORAL}~\cite{sun2016coral} is implemented here in a source-only form that aligns second-order feature statistics across the available source environments; no held-out-site statistics are used.
\end{itemize}

\textbf{Graph OOD methods.}
\begin{itemize}
    \item \textbf{GSAT}~\cite{miao2022gsat} is an interpretable and generalizable graph learning framework that, motivated by the information bottleneck principle, injects stochasticity into attention weights to select task-relevant subgraphs.
    \item \textbf{DisC}~\cite{fan2022disc} debiases GNNs under severely biased graph data by disentangling causal and bias substructures via an edge-mask generator and synthesizing counterfactual unbiased samples.
    \item \textbf{CEPG}~\cite{wang2026cepg} promotes invariant information learning via environment-promoted training objectives, extracting features that remain predictive across diverse graph environments for OOD generalization.
    \item \textbf{DiSCO}~\cite{sun2026disco} employs a diverse and sparse mixture-of-experts over causal subgraphs, handling instance-level heterogeneous causal patterns for OOD graph learning.
\end{itemize}

\textbf{Brain Networks.}
\begin{itemize}
    \item \textbf{STAGIN}~\cite{kim2021stagin} learns dynamic connectome representations using spatio-temporal attention over sequences of functional graphs.
    \item \textbf{BrainNetTF}~\cite{kan2022brainnttf} is a Transformer tailored to brain networks, using connection profiles as node features together with an orthonormal clustering readout aligned with ROI functional modules.
    \item \textbf{AGMGC}~\cite{noman2025agmgc} learns connectivity directly from raw fMRI time series through adaptive graph construction, and combines SplineCNN with multi-graph convolutions to capture localized and global dependencies for brain disorder classification.
    \item \textbf{FC-HGNN}~\cite{gu2025fchgnn} is a two-stage end-to-end heterogeneous graph neural network built on brain functional connectivity, applying type-aware message passing over a heterogeneous brain graph to identify mental disorders.
    \item \textbf{XG-GNN}~\cite{qiu2024towards} combines meta-learning with explainability-generalizable regularization to learn domain-generalizable diagnostic GNNs and explanation factors from fMRI BOLD signals.
    \item \textbf{LHDFormer}~\cite{xue2025lhdformer} models long-range and high-order dependencies in time-varying fMRI representations using adaptive embeddings and Transformer-based temporal modeling.
    \item \textbf{NeuroH-TGL}~\cite{li2025neurohtgl} learns temporal brain graphs while accounting for neuro-heterogeneity across subjects.
    \item \textbf{BrainOOD}~\cite{xu2025brainood} is an OOD-generalizable brain network framework that couples a feature selector, a structure extractor, and an improved graph information bottleneck objective to recover causal subgraphs predictive across unseen sites.
    \item \textbf{DeCI}~\cite{yu2026deci} moves beyond static functional connectivity by directly modeling raw fMRI time series, using cycle-and-drift decomposition with channel-independence for brain disorder classification.
\end{itemize}

\section{Implementation Details}
\label{app:implementation}
We implement \method{} in PyTorch and conduct all experiments on NVIDIA GeForce RTX 5090 GPUs. The observed covariates $\bm q_i$ comprise demographic variables (age and sex), head motion (mean framewise displacement), and acquisition variables (sequence length and TR), i.e., $d=5$; missing entries are imputed with the pooled training median. All covariates are standardized with pooled training statistics, and site-specific coefficients are fitted by Huber regression with $\delta=1.345$ using IRLS. Bootstrap sign stability is estimated with $B=50$ site resamples. Within each site the site-level covariate columns coincide with the intercept column; the IRLS updates use the pseudoinverse, and the site-wise residuals do not depend on how the intercept is apportioned among these columns. We use a two-layer GCN~\cite{kipf2016gcn} with hidden dimension 64 as the backbone, with gate prior strength $\lambda=0.3$ and gate temperature $\tau=1$, optimized using Adam with a learning rate of $1\times10^{-3}$, weight decay of $5\times10^{-4}$, batch size 64, and 100 epochs. Sliding-window FC is computed with window length $W=30$ and stride $S_w=5$ time points, so window duration scales with each site's TR. The target gate mass is set to $K=80$ for ABIDE, REST-meta-MDD, and SRPBS, and $K=150$ for ABCD, and the gate-budget regularization weight is $\beta=5\times10^{-4}$. In each LOSO fold, one site is held out for testing, two randomly selected source sites are reserved for validation, and the remaining sites are used for training. Site-specific deconfounders, scaffold statistics, source-reliability scores, and aggregation weights are estimated using only the training sites; validation sites are excluded from $\mathcal E_{\mathrm{tr}}$ and are therefore residualized with the frozen aggregate deconfounder, matching the test-time condition. Hyperparameters are tuned and the training epoch is selected on the validation sites. The scaffold thresholds are fixed to $\tau_E$ at the 80th percentile of the absolute consensus contrasts $\{|d_{\mathrm{com},j}|\}_{j\in\mathcal I}$ computed from the training sites, $\eta_E=0.75$, and $\zeta_E=0.70$ in all experiments; these values were chosen by grid search on the validation sites over the 70th--90th percentiles and $\eta_E,\zeta_E\in\{0.70,0.75,0.80,0.85,0.90\}$.  Results are reported as mean $\pm$ standard deviation across held-out test sites.

\section{Complexity Analysis}
\label{app:complexity}

Here we analyze the computational complexity of the proposed \method{}. The computational cost is mainly dominated by static FC computation, scaffold-based transient profiling, and gated message passing on the line graph. Let $N$, $P$, and $T$ denote the number of subjects, ROIs, and time points, respectively. Let $M=\mathcal{O}(P^2)$ be the number of full FC edges, $M_S=|S|$ be the number of scaffold edges, and $R=\lfloor (T-W)/S_w \rfloor+1$ be the number of sliding windows. $B_{\mathrm{sz}}$ is the mini-batch size, $D$ is the embedding dimension, $L$ is the number of GNN layers, and $N_m$ is the maximum number of line-graph neighbors per scaffold edge. Static FC computation and scaffold extraction take $\mathcal{O}(NTP^2+NM)$ time complexity, while transient profiling over scaffold edges introduces an additional cost of $\mathcal{O}(NRWM_S)$. During training, gated message passing on the sparse line graph takes $\mathcal{O}(LB_{\mathrm{sz}}M_SN_mD)$ per mini-batch. As a result, for $N_{\mathrm{ep}}$ training epochs, the overall computational complexity of \method{} is $\mathcal{O}\left(NTP^2+NM+NRWM_S+N_{\mathrm{ep}}LNM_SN_mD\right)$.
Since $N_m$ is small in practice, the trainable cost scales approximately linearly with the scaffold size $M_S$.

\section{Algorithm}

\input{table/algorithm}

The overall training and inference process of the proposed \method{} is shown in Algorithm~\ref{alg:overall}.

\section{More Experimental Results}
\label{app:results}
\subsection{ID Performance Comparison}
\label{sec:id_performance}

\input{table/ID}

Although \method{} is designed for cross-site OOD generalization, we also evaluate its predictive performance under in-distribution (ID) settings. Within each dataset, subjects from all sites are pooled and partitioned into 10 folds stratified by diagnostic label; in each of the 10 runs, one fold is held out for testing, the next two folds are used for validation, and the remaining seven folds are used for training, yielding a 7:2:1 split without holding out entire imaging sites as in the LOSO evaluation. Since every site is represented in the training folds, site-specific deconfounders and scaffold statistics are estimated for all sites from the training subjects of each run, and validation and test subjects are residualized with the frozen aggregate deconfounder as at inference in the LOSO evaluation. As shown in Table~\ref{tab:main_results_id}, \method{} achieves the highest mean AUC across all five evaluated settings and the highest mean ACC on REST-meta-MDD, SRPBS, and ABCD. It ranks second in ACC under both ABIDE parcellations, where DeCI achieves the highest mean. While the relative performance of the baselines varies across datasets and metrics, \method{} remains consistently competitive across the evaluated ASD, MDD, and ADHD-related tasks. These results demonstrate that the proposed framework is also effective for conventional subject-level prediction. Together with the LOSO results in the main text, the ID evaluation shows that \method{} combines improved unseen-site generalization with strong in-distribution classification performance.

\begin{figure}[t]
    \centering
    \captionsetup[subfigure]{font=small,skip=4pt}

    % Trim blank margins and use identical display dimensions.
    % Trim order: left, bottom, right, top.
    \newcommand{\appPanelImage}[2][0 0 0 0]{%
        \includegraphics[
            trim=#1,
            clip,
            width=\linewidth,
            height=3.0cm
        ]{#2}%
    }

    % Left column: (a,b).
    \begin{minipage}[t]{0.32\linewidth}
        \vspace{0pt}
        \centering
        \begin{subfigure}[t]{\linewidth}
            \centering
            \appPanelImage[42bp 7bp 42bp 1bp]
                {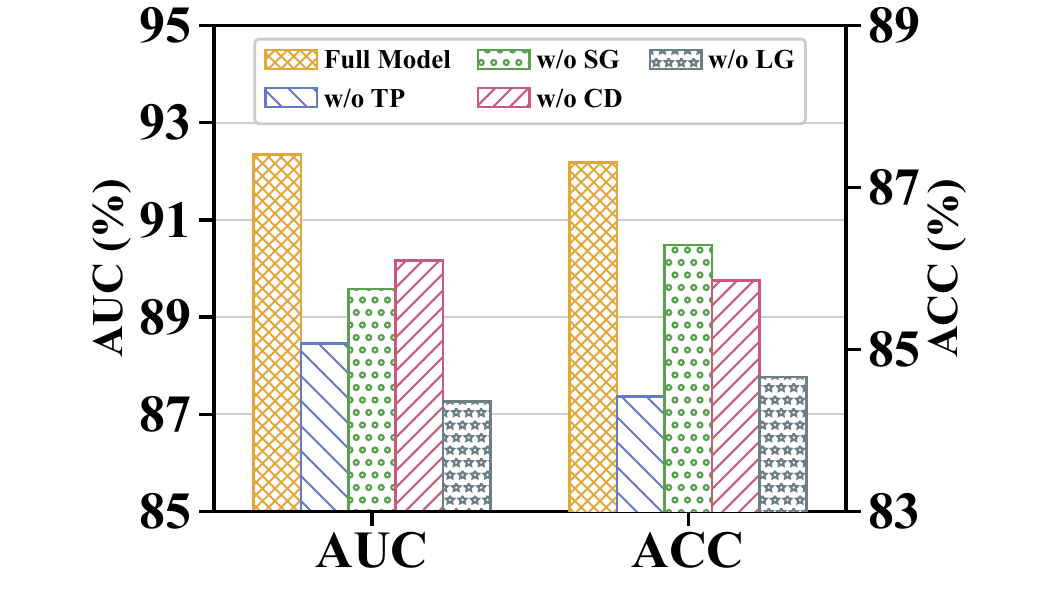}
            \caption{SRPBS}
            \label{fig:app_ablation_srpbs}
        \end{subfigure}

        \par\vspace{0.8em}

        \begin{subfigure}[t]{\linewidth}
            \centering
            \appPanelImage[42bp 7bp 42bp 1bp]
                {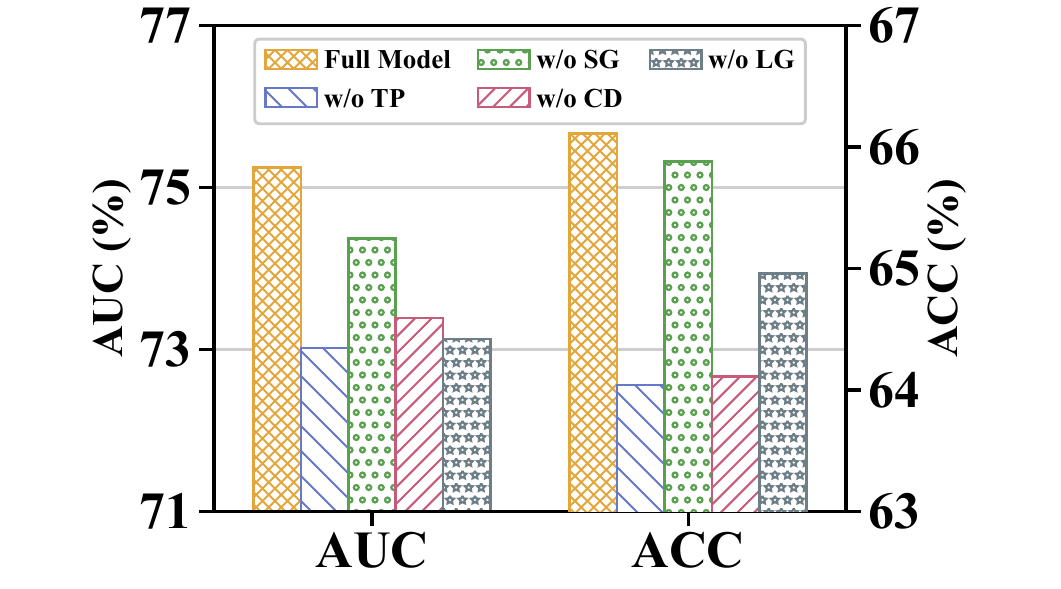}
            \caption{ABCD}
            \label{fig:app_ablation_abcd}
        \end{subfigure}
    \end{minipage}
    \hfill
    % Right columns: (c,d) above (e,f).
    \begin{minipage}[t]{0.66\linewidth}
        \vspace{0pt}
        \centering
        \begin{subfigure}[t]{0.485\linewidth}
            \centering
            \appPanelImage[1bp 11bp 2bp 4bp]
                {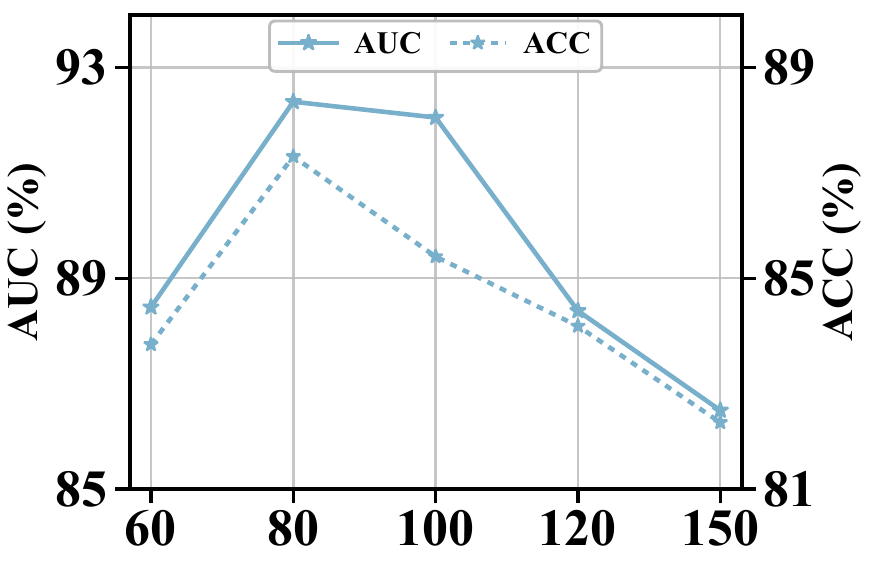}
            \caption{SRPBS $K$}
            \label{fig:app_k_srpbs}
        \end{subfigure}
        \hfill
        \begin{subfigure}[t]{0.485\linewidth}
            \centering
            \appPanelImage[1bp 11bp 2bp 4bp]
                {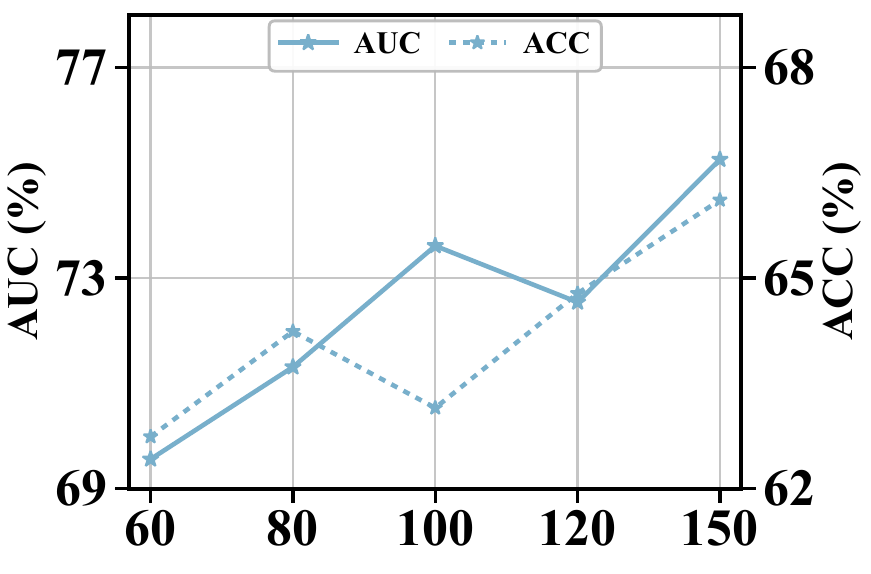}
            \caption{ABCD $K$}
            \label{fig:app_k_abcd}
        \end{subfigure}

        \par\vspace{0.8em}

        \begin{subfigure}[t]{0.485\linewidth}
            \centering
            \appPanelImage[13bp 6bp 2bp 24bp]
                {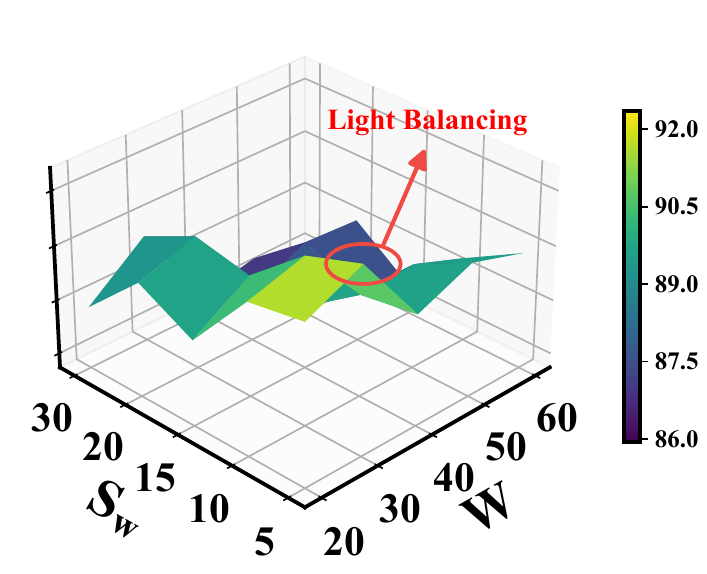}
            \caption{SRPBS $(W,S_w)$}
            \label{fig:app_window_srpbs}
        \end{subfigure}
        \hfill
        \begin{subfigure}[t]{0.485\linewidth}
            \centering
            \appPanelImage[16bp 5bp 5bp 23bp]
                {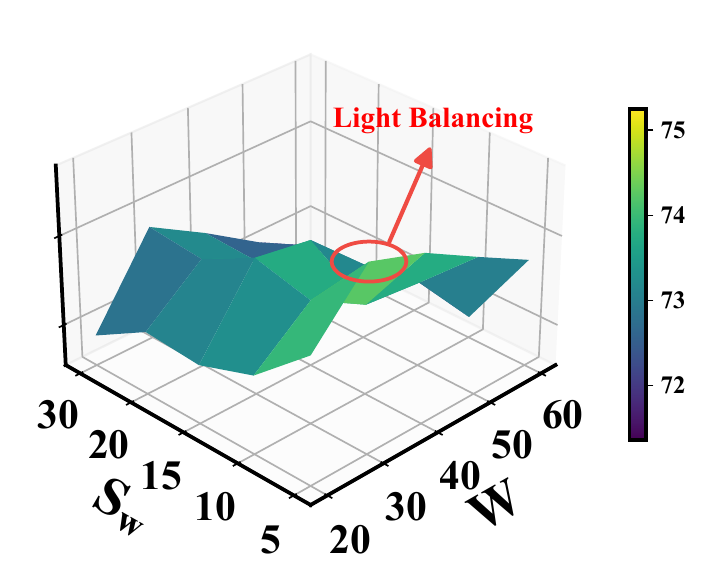}
            \caption{ABCD $(W,S_w)$}
            \label{fig:app_window_abcd}
        \end{subfigure}
    \end{minipage}

    \caption{Ablation and sensitivity analyses on SRPBS and ABCD:
    (a,b) component ablations;
    (c,d) sensitivity to gate budget $K$;
    (e,f) sensitivity to window length $W$ and stride $S_w$.}
    \label{fig:app_ablation}
\end{figure}

\subsection{More Ablation Study}
\label{app:ablation}
We further present ablation results of \method{} under OOD settings on SRPBS and ABCD in Fig.~\ref{fig:app_ablation}(a,b). On SRPBS, the full model achieves the highest mean AUC and ACC. Removing TP reduces performance, indicating that scaffold-level temporal descriptors provide complementary predictive information beyond static FC. Replacing site-aware deconfounding with global deconfounding also lowers performance, supporting the importance of accounting for site-dependent confounder effects. The declines observed for w/o LG and w/o SG highlight the contributions of pathway interactions and subject-adaptive weighting, respectively. Specifically, line-graph message passing models interactions among scaffold pathways sharing an ROI, while gating adjusts their contributions to individual predictions. On ABCD, the full model again achieves the highest mean AUC, with all ablated variants showing lower performance. The decreases for w/o TP and w/o CD support the contributions of temporal information and site-aware deconfounding, while w/o LG further demonstrates the benefit of modeling pathway interactions rather than encoding scaffold nodes independently. For ACC, w/o TP and w/o CD also fall below the full model, whereas w/o SG remains close to it. Thus, the benefit of subject-adaptive gating is more evident in AUC than in ACC on ABCD, while the contributions of individual components vary across datasets and metrics.

\begin{table}[t]
    \centering
    \caption{Controlled comparison of temporal representation and segmentation choices in \method{} on ABIDE under AAL and CC200 parcellations.}
    \label{tab:app_temporal_variants}
    \small
    \setlength{\tabcolsep}{5pt}
    \renewcommand{\arraystretch}{1.05}
    \resizebox{0.6\linewidth}{!}{%
    \begin{tabular}{@{}lcccc@{}}
    \toprule
    \multirow{2}{*}{Temporal variant}
    & \multicolumn{2}{c}{ABIDE}
    & \multicolumn{2}{c}{ABIDE (CC200)} \\
    \cmidrule(lr){2-3}\cmidrule(lr){4-5}
    & AUC & ACC & AUC & ACC \\
    \midrule
    \method{}-STD
    & 64.63$_{\pm 7.02}$ & 62.28$_{\pm 6.12}$
    & 62.85$_{\pm 8.44}$ & 61.82$_{\pm 8.94}$ \\
    \method{}-MASD
    & 64.78$_{\pm 6.97}$ & 62.41$_{\pm 6.08}$
    & 62.94$_{\pm 8.38}$ & 61.95$_{\pm 8.91}$ \\
    \method{}-GRU
    & \textbf{65.71}$_{\pm 6.85}$ & 62.77$_{\pm 5.99}$
    & 63.49$_{\pm 8.29}$ & 61.96$_{\pm 8.87}$ \\
    \method{}-CP
    & 65.22$_{\pm 6.91}$ & 63.05$_{\pm 5.98}$
    & 63.31$_{\pm 8.33}$ & 62.15$_{\pm 8.89}$ \\
    \midrule
    \method{}
    & 65.65$_{\pm 6.42}$ & \textbf{64.83$_{\pm 6.62}$}
    & \textbf{64.61$_{\pm 7.12}$} & \textbf{62.18$_{\pm 8.35}$} \\
    \bottomrule
    \end{tabular}%
    }
\end{table}

Additionally, we evaluate four variants of transient pathway profiling on ABIDE under the AAL and CC200 parcellations. \method{}-STD retains temporal volatility but removes temporal range, whereas \method{}-MASD replaces range with the mean absolute successive difference between consecutive windowed FC estimates, incorporating local temporal ordering while preserving feature dimensionality. \method{}-GRU replaces both handcrafted descriptors with a GRU encoding of the FC trajectory on each scaffold edge, while \method{}-CP replaces fixed sliding windows with change-point-defined segments. As shown in Table~\ref{tab:app_temporal_variants}, combining volatility and range consistently outperforms volatility alone across both metrics and parcellations, supporting the complementary contribution of temporal range. MASD provides only marginal improvements over volatility alone and remains below the volatility--range configuration. The GRU encoder achieves a slightly higher AUC under AAL but lower ACC, and trails the reference configuration on both metrics under CC200. Change-point segmentation also yields lower mean performance under both parcellations. Thus, incorporating local temporal ordering, sequence encoding, or adaptive segmentation does not provide consistent gains over the volatility--range configuration. These results support the selected descriptors as a compact temporal representation for the evaluated cross-site prediction tasks.

\begin{table}[t]
    \centering
    \caption{Controlled ablation of the scaffold prior and subject-conditioned term in the gate logits on ABIDE under AAL and CC200 parcellations.}
    \label{tab:app_gating_components}
    \small
    \setlength{\tabcolsep}{5pt}
    \renewcommand{\arraystretch}{1.05}
    \resizebox{0.6\linewidth}{!}{%
    \begin{tabular}{@{}lcccc@{}}
    \toprule
    \multirow{2}{*}{Gating variant}
    & \multicolumn{2}{c}{ABIDE}
    & \multicolumn{2}{c}{ABIDE (CC200)} \\
    \cmidrule(lr){2-3}\cmidrule(lr){4-5}
    & AUC & ACC & AUC & ACC \\
    \midrule
    Prior-only
    & $62.67_{\pm 6.34}$ & $60.83_{\pm 5.51}$
    & $61.14_{\pm 7.88}$ & $59.86_{\pm 8.42}$ \\
    Subject-only
    & $64.12_{\pm 7.18}$ & $62.19_{\pm 6.12}$
    & $62.51_{\pm 8.54}$ & $61.42_{\pm 9.06}$ \\
    \midrule
    \method{}
    & \textbf{65.65$_{\pm 6.42}$} & \textbf{64.83$_{\pm 6.62}$}
    & \textbf{64.61$_{\pm 7.12}$} & \textbf{62.18$_{\pm 8.35}$} \\
    \bottomrule
    \end{tabular}%
    }
    \vspace{-0.3cm}
\end{table}

We next examine the contributions of population guidance and subject-conditioned adaptation to the gating mechanism on ABIDE under the AAL and CC200 parcellations. Prior-only retains the scaffold-derived prior but removes the subject-conditioned term from the gate logits, whereas Subject-only removes the prior term and computes gates from local pathway features and global context. As shown in Table~\ref{tab:app_gating_components}, Subject-only outperforms Prior-only under both parcellations, highlighting the value of adapting pathway weights to individual connectivity features. Combining the two terms further improves performance, yielding the highest mean AUC and ACC under both parcellations. These results support the complementary roles of scaffold-derived population guidance and subject-conditioned pathway weighting in cross-site prediction.

\begin{table}[t]
\centering
\caption{Comparison of \method{} w/o TP with representative static-FC baselines under LOSO evaluation. \textbf{Bold} denotes the best performance.}
\label{tab:app_static_core_comparison}
\small
\setlength{\tabcolsep}{4pt}
\renewcommand{\arraystretch}{1.05}
\resizebox{\linewidth}{!}{%
\begin{tabular}{@{}lcccccccc@{}}
\toprule
\multirow{2}{*}{Method}
& \multicolumn{2}{c}{ABIDE}
& \multicolumn{2}{c}{REST-meta-MDD}
& \multicolumn{2}{c}{SRPBS}
& \multicolumn{2}{c}{ABCD} \\
\cmidrule(lr){2-3}\cmidrule(lr){4-5}
\cmidrule(lr){6-7}\cmidrule(lr){8-9}
& AUC & ACC & AUC & ACC & AUC & ACC & AUC & ACC \\
\midrule
DisC
& 56.67$_{\pm 7.85}$ & 56.14$_{\pm 8.68}$
& 56.21$_{\pm 4.42}$ & 53.28$_{\pm 3.88}$
& 73.72$_{\pm 19.78}$ & 69.97$_{\pm 16.93}$
& 69.59$_{\pm 7.70}$ & 58.83$_{\pm 9.45}$ \\
BrainNetTF
& 62.85$_{\pm 6.20}$ & 60.90$_{\pm 7.06}$
& 61.50$_{\pm 7.40}$ & 58.20$_{\pm 5.60}$
& 61.56$_{\pm 17.03}$ & 60.41$_{\pm 10.80}$
& 68.28$_{\pm 4.54}$ & 56.91$_{\pm 10.43}$ \\
FC-HGNN
& 55.26$_{\pm 7.90}$ & 52.23$_{\pm 4.58}$
& 58.69$_{\pm 6.69}$ & 56.78$_{\pm 5.71}$
& 57.30$_{\pm 10.60}$ & 53.80$_{\pm 6.40}$
& 71.08$_{\pm 5.74}$ & 59.10$_{\pm 12.34}$ \\
BrainOOD
& 63.82$_{\pm 6.73}$ & 59.66$_{\pm 9.70}$
& 63.22$_{\pm 6.29}$ & 58.45$_{\pm 5.20}$
& \textbf{90.00$_{\pm 7.33}$} & 77.70$_{\pm 15.23}$
& 72.77$_{\pm 8.07}$ & \textbf{64.36$_{\pm 7.24}$} \\
\midrule
\method{} w/o TP
& \textbf{65.12$_{\pm 7.91}$} & \textbf{62.08$_{\pm 8.84}$}
& \textbf{65.95$_{\pm 9.04}$} & \textbf{64.12$_{\pm 5.56}$}
& 88.46$_{\pm 9.61}$ & \textbf{84.42$_{\pm 11.24}$}
& \textbf{73.02$_{\pm 9.02}$} & 64.04$_{\pm 8.65}$ \\
\bottomrule
\end{tabular}%
}
\end{table}

Finally, we compare \method{} w/o TP with representative static-FC baselines on four datasets under LOSO evaluation. This variant removes temporal descriptors. As shown in Table~\ref{tab:app_static_core_comparison}, \method{} w/o TP achieves the highest mean AUC and ACC on ABIDE and REST-meta-MDD. On SRPBS, it attains the highest mean ACC but trails BrainOOD in AUC, whereas on ABCD, it leads in AUC and ranks second in ACC. These results show that the retained framework remains competitive across datasets using static FC, indicating that its predictive performance is not solely attributable to additional temporal information.

\subsection{More Sensitivity Study}
\label{app:sensitivity}
We also present the sensitivity study of \method{} under OOD settings on the SRPBS and ABCD datasets. As shown in Fig.~\ref{fig:app_ablation}(c), SRPBS exhibits a clear inverted-U trend as $K$ increases. Performance peaks around $K=80$, while both smaller and larger budgets degrade results. This pattern suggests that a small budget may overly suppress informative pathways, whereas a large budget may increase the contributions of less informative pathways. On ABCD, Fig.~\ref{fig:app_ablation}(d) shows that performance generally improves with increasing $K$, with the best result at the largest evaluated budget. This suggests that ABCD benefits from broader pathway coverage, likely due to greater inter-subject variability and more distributed diagnostic patterns. Nevertheless, small budgets still underperform, indicating that insufficient pathway coverage limits subject-adaptive message passing. For temporal profiling, Fig.~\ref{fig:app_ablation}(e,f) shows that both SRPBS and ABCD favor moderate windowing with relatively fine temporal strides. Coarse temporal partitioning degrades performance, indicating that transient connectivity dynamics are informative for cross-site generalization, while overly aggressive segmentation may yield unstable dynamic FC estimates.

\begin{table}[t]
    \centering
    \caption{Sensitivity to scaffold-selection criteria on ABIDE under AAL and CC200 parcellations. \textbf{Bold} settings denote defaults.}
    \label{tab:app_scaffold_sensitivity}
    \small
    \setlength{\tabcolsep}{5pt}
    \renewcommand{\arraystretch}{1.05}
    \resizebox{0.75\linewidth}{!}{%
    \begin{tabular}{@{}lccccc@{}}
    \toprule
    \multirow{2}{*}{Criterion}
    & \multirow{2}{*}{Setting}
    & \multicolumn{2}{c}{ABIDE (AAL)}
    & \multicolumn{2}{c}{ABIDE (CC200)} \\
    \cmidrule(lr){3-4}\cmidrule(lr){5-6}
    & & AUC & ACC & AUC & ACC \\
    \midrule
    \multirow{5}{*}{Percentile for $\tau_E$}
    & 70th
    & 63.77$_{\pm 7.12}$ & 61.29$_{\pm 6.24}$
    & 62.11$_{\pm 8.51}$ & 60.85$_{\pm 9.02}$ \\
    & 75th
    & 64.82$_{\pm 6.95}$ & 62.65$_{\pm 6.08}$
    & 63.15$_{\pm 8.33}$ & 61.92$_{\pm 8.91}$ \\
    & \textbf{80th}
    & 65.65$_{\pm 6.42}$ & 64.83$_{\pm 6.62}$
    & 64.61$_{\pm 7.12}$ & 62.18$_{\pm 8.35}$ \\
    & 85th
    & 65.21$_{\pm 6.91}$ & 63.12$_{\pm 5.96}$
    & 63.55$_{\pm 8.29}$ & 62.45$_{\pm 8.85}$ \\
    & 90th
    & 64.55$_{\pm 7.05}$ & 62.88$_{\pm 6.11}$
    & 63.02$_{\pm 8.42}$ & 61.88$_{\pm 8.98}$ \\
    \midrule
    \multirow{3}{*}{Sign consistency $\eta_E$}
    & 0.70
    & 65.15$_{\pm 6.92}$ & 63.22$_{\pm 5.94}$
    & 63.45$_{\pm 8.31}$ & 62.35$_{\pm 8.88}$ \\
    & \textbf{0.75}
    & 65.65$_{\pm 6.42}$ & 64.83$_{\pm 6.62}$
    & 64.61$_{\pm 7.12}$ & 62.18$_{\pm 8.35}$ \\
    & 0.90
    & 64.41$_{\pm 7.11}$ & 62.15$_{\pm 6.14}$
    & 62.85$_{\pm 8.45}$ & 61.70$_{\pm 9.06}$ \\
    \midrule
    \multirow{5}{*}{Bootstrap stability $\zeta_E$}
    & \textbf{0.70}
    & 65.65$_{\pm 6.42}$ & 64.83$_{\pm 6.62}$
    & 64.61$_{\pm 7.12}$ & 62.18$_{\pm 8.35}$ \\
    & 0.75
    & 65.41$_{\pm 6.90}$ & 63.32$_{\pm 5.93}$
    & 63.58$_{\pm 8.28}$ & 62.55$_{\pm 8.85}$ \\
    & 0.80
    & 65.25$_{\pm 6.93}$ & 63.10$_{\pm 5.97}$
    & 63.42$_{\pm 8.32}$ & 62.20$_{\pm 8.89}$ \\
    & 0.85
    & 64.88$_{\pm 7.01}$ & 62.75$_{\pm 6.06}$
    & 63.10$_{\pm 8.38}$ & 61.85$_{\pm 8.95}$ \\
    & 0.90
    & 64.35$_{\pm 7.09}$ & 62.25$_{\pm 6.17}$
    & 62.65$_{\pm 8.49}$ & 61.35$_{\pm 9.01}$ \\
    \bottomrule
    \end{tabular}%
    }
    \vspace{-0.4cm}
\end{table}

We further examine sensitivity to the scaffold-selection criteria on ABIDE under AAL and CC200 parcellations. The contrast threshold $\tau_E$, sign-consistency cutoff $\eta_E$, and bootstrap-stability cutoff $\zeta_E$ are varied independently, with the other two fixed at their default values determined through source-side validation. The default configuration uses the 80th percentile of the absolute consensus contrasts, $\eta_E=0.75$, and $\zeta_E=0.70$. We vary the contrast percentile from the 70th to the 90th in increments of five and evaluate $\zeta_E\in\{0.70,0.75,0.80,0.85,0.90\}$. Since each fold has seven training sites, $\kappa_j$ takes values in multiples of $1/7$, and cutoffs in $[0.75,0.85]$ select identical scaffolds; we therefore report the three distinct settings $\eta_E\in\{0.70,0.75,0.90\}$, corresponding to sign agreement at no fewer than five, six, and all seven training sites. As shown in Table~\ref{tab:app_scaffold_sensitivity}, the default configuration achieves the highest mean AUC under both parcellations and the highest mean ACC under AAL. ACC under CC200 varies modestly across neighboring settings, with the best setting depending on the selection criterion. Performance declines at the most restrictive end of each tested range, indicating that overly stringent filtering can exclude informative pathways.

\begin{table}[t]
    \centering
    \caption{Within- and between-group gate similarity and within-group gate variation on held-out ABIDE subjects.}
    \label{tab:app_gate_behavior}
    \small
    \setlength{\tabcolsep}{6pt}
    \renewcommand{\arraystretch}{1.05}
    \resizebox{0.75\linewidth}{!}{%
    \begin{tabular}{@{}ccccc@{}}
    \toprule
    $S_{\mathrm{within}}$ & $S_{\mathrm{between}}$
    & $\Delta S$ & $\Delta J$ & $R_{\mathrm{within}}$ \\
    \midrule
    $0.9649_{\pm 0.0155}$ & $0.9643_{\pm 0.0160}$
    & $0.0005_{\pm 0.0010}$ & $0.0020_{\pm 0.0028}$
    & $0.9697_{\pm 0.0069}$ \\
    \bottomrule
    \end{tabular}%
    }
\end{table}
\input{table/top10_edge}
To complement these analyses, we examine how the learned gates vary across held-out ABIDE subjects. Held-out diagnostic labels are used only for this post-hoc analysis, and all statistics are computed within each held-out site before aggregation across LOSO folds. Let $\bm g_i$ denote the gate vector of subject $i$. We compute the mean cosine similarities between subject pairs within the same diagnostic group and across different groups, denoted by $S_{\mathrm{within}}$ and $S_{\mathrm{between}}$, respectively, and define $\Delta S=S_{\mathrm{within}}-S_{\mathrm{between}}$. To compare the continuous pathway weights, we also compute weighted Jaccard similarity:
\begin{equation}
    J_{\mathrm{w}}(\bm g_i,\bm g_k)
    =
    \frac{\sum_p\min(g_{i,p},g_{k,p})}
         {\sum_p\max(g_{i,p},g_{k,p})}.
    \label{eq:app_gate_weighted_jaccard}
\end{equation}
We analogously define $\Delta J$ as the within-minus-between difference in mean weighted Jaccard similarity. The fraction of total gate variation occurring within diagnostic groups is measured by
\begin{equation}
    R_{\mathrm{within}}
    =
    \frac{
        \sum_{y\in\{0,1\}}\sum_{i:Y_i=y}
        \|\bm g_i-\bar{\bm g}_y\|_2^2
    }{
        \sum_i\|\bm g_i-\bar{\bm g}\|_2^2
    },
    \label{eq:app_gate_within_variation}
\end{equation}
where $\bar{\bm g}_y$ and $\bar{\bm g}$ are the group-specific and overall mean gate vectors within the held-out site. Table~\ref{tab:app_gate_behavior} shows high cosine and weighted Jaccard similarity both within and between diagnostic groups, with small within-minus-between differences under both measures, and most gate variation occurs within diagnostic groups rather than between their mean gate vectors. The high similarity across all subject pairs reflects the shared component introduced by the scaffold prior, which up-weights pathways with large consensus contrasts for every subject; the gating ablations in Table~\ref{tab:app_gating_components} confirm that this population weighting alone already improves over uniform gates. The near-zero group differences indicate that the remaining subject-level variation is not aligned with the diagnostic label, so the gates do not act as a shortcut classifier. Instead, the subject-conditioned term redistributes weight among scaffold pathways according to each subject's connectivity profile, while discrimination is carried out by the subsequent message passing and readout. This modulation is nonetheless beneficial, as removing the subject-conditioned term consistently lowers performance relative to the full model, and removing gating altogether leads to larger drops in the ablation study. Together, these results indicate that the population prior provides a shared pathway weighting that is stable across sites, and the subject-conditioned term supplies a modest but consistent individual adjustment on top of it.

\subsection{More Interpretability Analysis}
\label{sec:interpre}

\textbf{Top Positive Scaffold Edges on ABIDE.}
In this part, to assess biological plausibility and the stability of scaffold selection, we additionally repeat each ABIDE LOSO fold with different random seeds, whereas the main results use a single run per fold. Specifically, we summarize scaffold edges across the $E=10$ LOSO folds, each repeated for $R=10$ runs; for fold $e$ and run $r$, let $\mathcal S^{(e,r)}$ denote the population scaffold estimated from the source sites. Across the resulting $ER=100$ scaffold constructions, the selection frequency of edge $j$ is
\begin{equation}
\label{eq:selection_frequency_interpret}
    \nu_j
    =
    \frac{1}{ER}
    \sum_{e=1}^{E}\sum_{r=1}^{R}
    \mathbb I\!\left[j\in\mathcal S^{(e,r)}\right].
\end{equation}

For each selected edge, we first average its continuous gate over subjects from the corresponding held-out site:
\begin{equation}
\label{eq:fold_gate_interpret}
    g_j^{(e,r)}
    =
    \frac{1}{N_e}
    \sum_{i\in\mathcal D_e^{\mathrm{test}}}
    g_{i,j}^{(e,r)},
\end{equation}
where $\mathcal D_e^{\mathrm{test}}$ contains $N_e$ test subjects and $g_{i,j}^{(e,r)}$ denotes the gate of the line-graph node corresponding to edge $j$. The Gate value in Table~\ref{tab:avg_edges} averages these site-level means over the evaluations in which the edge is selected:
\begin{equation}
\label{eq:mean_gate_interpret}
    \mathrm{Gate}_j
    =
    \frac{
        \sum_{(e,r):\,j\in\mathcal S^{(e,r)}} g_j^{(e,r)}
    }{
        \sum_{e=1}^{E}\sum_{r=1}^{R}
        \mathbb I\!\left[j\in\mathcal S^{(e,r)}\right]
    }.
\end{equation}
Thus, $\nu_j$ measures recurrence across folds and runs, whereas Gate summarizes pathway weighting for held-out subjects conditional on selection. Each selected fold and run contributes equally to the reported Gate value.

Let $\kappa_j^{(e,r)}$ denote the source-site sign consistency defined in Sec.~\ref{sec:decoupling} for fold $e$ and run $r$. The sign consistency reported in Table~\ref{tab:avg_edges} is its average across all evaluations:
\begin{equation}
\label{eq:mean_kappa_interpret}
    \kappa_j
    =
    \frac{1}{ER}
    \sum_{e=1}^{E}\sum_{r=1}^{R}
    \kappa_j^{(e,r)}.
\end{equation}
We similarly average the consensus diagnostic contrasts to obtain
\begin{equation}
\label{eq:mean_dcom_interpret}
    d_{\mathrm{com},j}
    =
    \frac{1}{ER}
    \sum_{e=1}^{E}\sum_{r=1}^{R}
    d_{\mathrm{com},j}^{(e,r)}.
\end{equation}
An edge is termed positive when $d_{\mathrm{com},j}>0$. Because the diagnostic contrast is defined as ASD minus TD, positive values indicate higher residualized FC in ASD than in TD on average across evaluations. Positivity therefore refers to the direction of the label-associated contrast, not the sign of the raw Pearson correlation.

Table~\ref{tab:avg_edges} reports the ten positive scaffold edges with the largest $|d_{\mathrm{com},j}|$, together with their average gate values, selection frequencies, and sign consistencies. These connections involve association, sensory, and subcortical regions, including inferior parietal and angular gyri, thalamic and temporal areas, occipital cortex, insula, putamen, and pallidum. Their distribution is broadly consistent with prior ASD resting-state connectivity findings involving large-scale association networks, thalamocortical circuits, the salience network, and striatal systems~\cite{hull2017resting,lee2016abnormalities,di2011aberrant}. Most reported edges exhibit high selection frequencies and strong sign consistency, indicating that these label-associated connections recur across LOSO folds and repeated runs.

\textbf{Site-wise Scaffold Edges on ABIDE.}

\begin{table}[t]
\centering
\caption{Time consumption of different methods in the training stage for each epoch (in seconds).}
\label{tab:time}
\resizebox{0.55\linewidth}{!}{
\begin{tabular}{lcccc}
\toprule
Methods & ABIDE & REST-meta-MDD & SRPBS & ABCD \\
\midrule
BrainOOD& 0.65 & 2.48 & 0.82 & 0.58 \\
BrainNetTF & 0.35 & 1.25 & 0.45 & 0.32 \\
AGMGC& 0.32 & 1.18 & 0.42 & 0.30 \\
FC-HGNN& 0.15 & 0.35 & 0.18 & 0.12 \\
XG-GNN& 0.28 & 0.98 & 0.35 & 0.25 \\
\method{}& 0.15 & 0.38 & 0.20 & 0.12 \\
\bottomrule
\end{tabular}
}
\vspace{-0.2cm}
\end{table}

\begin{table}[t]
\centering
\caption{GPU memory consumption of different methods in the training stage (in GB).}
\label{tab:memory}
\resizebox{0.55\linewidth}{!}{
\begin{tabular}{lcccc}
\toprule
Methods & ABIDE & REST-meta-MDD & SRPBS & ABCD \\
\midrule
BrainOOD& 1.9 & 3.5 & 2.2 & 2.0 \\
BrainNetTF & 1.5 & 2.8 & 1.8 & 1.6 \\
AGMGC& 1.4 & 2.6 & 1.7 & 1.5 \\
FC-HGNN& 1.4 & 2.7 & 1.7 & 1.5 \\
XG-GNN& 1.6 & 3.0 & 1.9 & 1.7 \\
\method{}& 1.5 & 3.1 & 1.9 & 1.6 \\
\bottomrule
\end{tabular}
}
\vspace{-0.4cm}
\end{table}

To examine how the selected connectivity patterns vary across LOSO folds, Fig.~\ref{fig:site_edges} shows, for each held-out site, the ten scaffold edges with the largest $|d^{(e,1)}_{\mathrm{com},j}|$ in the scaffold estimated from the remaining source sites (run $r=1$), alongside multiple anatomical views. Although the exact ROI pairs differ across folds, the selected edges repeatedly involve association, temporal, parietal, frontal, cingulate, occipital, and subcortical regions.

\begin{figure}[h]
    \centering
    \vspace{-0.3cm}
    \includegraphics[width=0.87\linewidth]{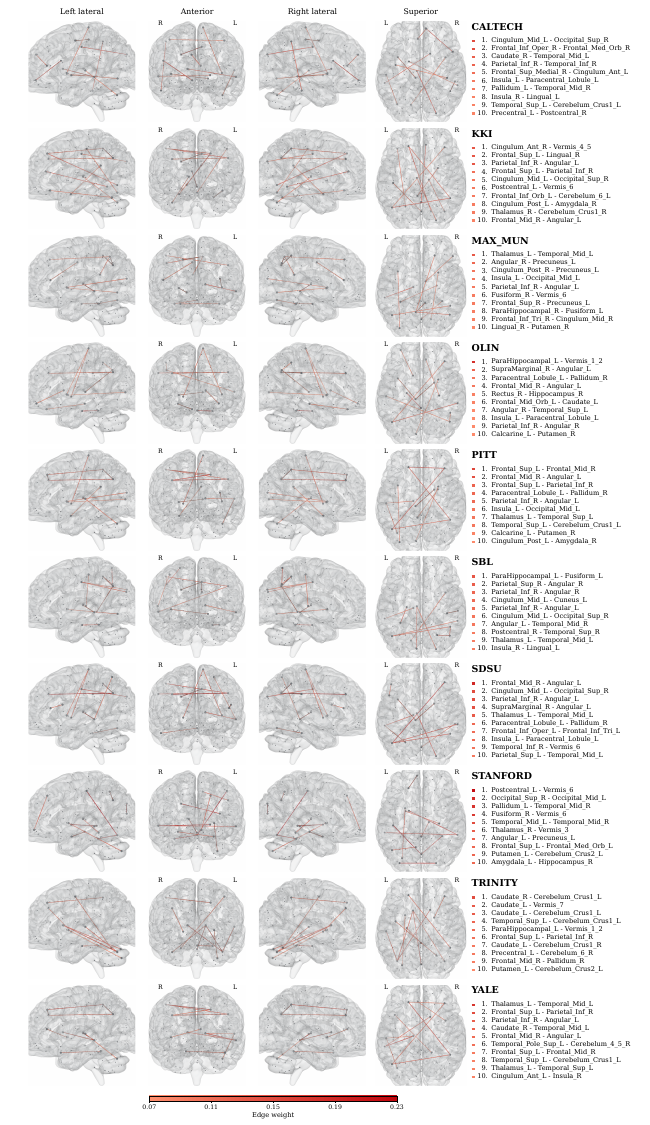}
    \caption{Scaffold edges for each held-out ABIDE site under the OOD setting. Each row lists the ten ROI pairs with the largest $|d^{(e,1)}_{\mathrm{com},j}|$ in the scaffold estimated from the remaining source sites (run $r=1$).}
    \label{fig:site_edges}
    \vspace{-0.2cm}
\end{figure}

\subsection{Efficiency and Resource Consumption Analysis}\label{sec:time_and_gpu}

We further evaluate the training efficiency of different methods in terms of per-epoch training time and GPU memory consumption. As shown in Tables~\ref{tab:time} and~\ref{tab:memory}, \method{} achieves competitive efficiency across all datasets. In terms of training time, \method{} is consistently faster than BrainOOD~\cite{xu2025brainood}, BrainNetTF~\cite{kan2022brainnttf}, AGMGC~\cite{noman2025agmgc}, and XG-GNN~\cite{qiu2024towards}, while remaining comparable to the lightweight FC-HGNN baseline. This efficiency stems from performing message passing on a compact scaffold-level graph, which avoids redundant computation over dense brain connectivity. For GPU memory consumption, \method{} uses less memory than BrainOOD across all datasets and remains broadly comparable to XG-GNN and the lightweight baselines.

%% file: table/notations.tex
\begin{table}[h]
\centering
\small
\caption{Key notation used in the main text.}
\label{tab:notation}
\begin{tabular}{p{0.2\linewidth}p{0.71\linewidth}}
\toprule
Symbol & Meaning \\
\midrule
$e\in\mathcal E_{\mathrm{tr}},i$
& Source-site and subject indices. \\

$(\bm X_i,\bm q_i,Y_i)$
& BOLD time series, observed covariates, and binary label of subject $i$. \\

$j\in\mathcal I$
& Canonical FC-edge index corresponding to ROI pair $(u_j,v_j)$. \\

$\bm r_i,\tilde{\bm q}_i$
& Full-scan static FC vector and covariates standardized using source data. \\

$\hat{\bm b}_e,\hat{\bm\Gamma}_e,\hat{\bm\theta}_e$
& Estimated edge-wise intercepts, covariate coefficients, and label coefficients at source site $e$; only the first two are subtracted. \\

$u_e,\alpha_e$
& Cross-source disagreement score and reliability weight of deconfounder $e$. \\

$\bar{\bm b}_w,\bar{\bm\Gamma}_w$
& Reliability-weighted aggregate coefficients for unseen-site correction. \\

$\bm r_i^{\mathrm{site}},\bm r_i^{\mathrm{res}}$
& Static FC after source-site correction and frozen aggregate correction, respectively. \\

$\bm r_i^{\mathrm{dec}}$
& Static FC used for representation:
$\bm r_i^{\mathrm{site}}$ for training-site subjects and
$\bm r_i^{\mathrm{res}}$ otherwise. \\

$\bm d_e,\bm d_{\mathrm{com}}$
& Site-level label contrasts and their component-wise median across source sites. \\

$\kappa_j,\pi_j$
& Cross-site sign consistency and bootstrap sign stability of FC edge $j$. \\

$\tau_E,\eta_E,\zeta_E,B$
& Contrast, sign-consistency, and bootstrap-stability thresholds for scaffold selection, and the number of bootstrap replicates. \\

$\mathcal S,M_{\mathcal S}$
& Scaffold of reproducible label-associated FC edges and its size
$M_{\mathcal S}=|\mathcal S|$. \\

$p,p',j_p$
& Line-graph node indices $p,p'$ and the original FC-edge index $j_p$
corresponding to node $p$. \\

$W,S_w,L_i$
& Sliding-window length, stride, and number of windows for subject $i$. \\

$c_{i,j}^{(t)}$
& Fisher-$z$-transformed FC of edge $j$ in window $t$ for subject $i$. \\

$s_{i,j},f_{i,j}$
& Temporal volatility and range of the windowed Fisher-$z$ FC trajectory. \\

$\bm h_{i,p}^{(0)},\bm h_{i,p}^{(l)}$
& Initial node features and layer-$l$ representation of scaffold node $p$
for subject $i$. \\

$\bm\mu_0$
& Population prior scores obtained by standardizing absolute consensus
contrasts on scaffold edges. \\

$\tilde{\bm A}$
& Symmetrically normalized line-graph propagation matrix with self-loops. \\

$\bm c_i$
& Subject context computed from pooled initial node features. \\

$g_{i,p}$
& Continuous subject-adaptive gate for scaffold node $p$. \\

$\bm z_i$
& Subject-level representation obtained through gated readout. \\

$\lambda,\tau,\beta$
& Prior strength and temperature of the gate, and weight of the gate-budget penalty. \\

$K$
& Target total mass of the continuous gates. \\
\bottomrule
\end{tabular}
\end{table}

%% file: table/dataset_description.tex
\begin{table}[t]
    \centering
    \caption{Demographic and clinical characteristics of the ABIDE, REST-meta-MDD, SRPBS, and ABCD datasets after site filtering. ASD, TD, MDD, ADHD, and HC denote autism spectrum disorder, typically developing, major depressive disorder, attention-deficit/hyperactivity disorder, and healthy control, respectively.}
    \vspace{0.2cm}
    \label{tab:original_data}
    \resizebox{\textwidth}{!}{%
    \begin{tabular}{lcccc}
    \toprule
    \textbf{Characteristic} 
    & \textbf{ABIDE} 
    & \textbf{REST-meta-MDD} 
    & \textbf{SRPBS} 
    & \textbf{ABCD (ADHD-task)} \\
    \midrule
    Sample size 
    & 435 
    & 1{,}596 
    & 968 
    & 425 \\
    
    Class composition 
    & 209 ASD / 226 TD 
    & 845 MDD / 751 HC 
    & 484 MDD / 484 HC 
    & 213 ADHD / 212 HC \\
    
    Age (mean $\pm$ std) 
    & $18.15 \pm 9.53$ 
    & $39.28 \pm 16.08$ 
    & $42.93 \pm 13.58$ 
    & $9.47 \pm 0.50$ \\
    
    Sex (M/F) 
    & 367/68 
    & 629/967 
    & 437/531 
    & 246/179 \\
    
    \bottomrule
    \end{tabular}%
    }
    \vspace{-0.2cm}
\end{table}
    
\begin{table}[t]
    \centering
    \caption{Statistics of the ABIDE dataset. $T$ denotes the number of fMRI time points.}
    \vspace{0.2cm}
    \label{tab:abide_data}
    \begingroup
    \setlength{\tabcolsep}{4pt}
    \resizebox{0.6\textwidth}{!}{%
    \begin{tabular}{lrrrccr}
    \toprule
    Site & $N$ & TD & ASD & Age (mean $\pm$ std) & Male/Female & $T$ \\
    \midrule
    CALTECH  & 37  & 18  & 19  & $27.7 \pm 10.3$ & 29/8  & 146 \\
    KKI      & 48  & 28  & 20  & $10.0 \pm 1.3$  & 36/12 & 148 \\
    MAX\_MUN & 52  & 28  & 24  & $25.3 \pm 11.8$ & 48/4  & 140 \\
    OLIN     & 34  & 15  & 19  & $16.6 \pm 3.4$  & 29/5  & 206 \\
    PITT     & 56  & 27  & 29  & $18.9 \pm 6.9$  & 48/8  & 196 \\
    SBL      & 30  & 15  & 15  & $34.4 \pm 8.5$  & 30/0  & 196 \\
    SDSU     & 36  & 22  & 14  & $14.4 \pm 1.8$  & 29/7  & 176 \\
    STANFORD & 39  & 20  & 19  & $10.0 \pm 1.6$  & 31/8  & 209 \\
    TRINITY  & 47  & 25  & 22  & $17.0 \pm 3.4$  & 47/0  & 146 \\
    YALE     & 56  & 28  & 28  & $12.7 \pm 2.9$  & 40/16 & 196 \\
    \midrule
    \textbf{Total} & \textbf{435} & \textbf{226} & \textbf{209} &  & \textbf{367/68} &  \\
    \bottomrule
    \end{tabular}%
    }
    \vspace{-0.2cm}
    \endgroup
\end{table}

\begin{table}[t]
    \centering
    \caption{Statistics of the REST-meta-MDD dataset. $T$ denotes the number of fMRI time points.}
    \vspace{0.2cm}
    \label{tab:mdd_data}
    \begingroup
    \setlength{\tabcolsep}{4pt}
    \resizebox{0.6\textwidth}{!}{%
    \begin{tabular}{lrrrccr}
    \toprule
    Site & $N$ & HC & MDD & Age (mean $\pm$ std) & Male/Female & $T$ \\
    \midrule
    S1  & 148 & 74  & 74  & $31.8 \pm 8.5$  & 63/85   & 200 \\
    S10 & 83  & 33  & 50  & $31.9 \pm 11.0$ & 42/41   & 202 \\
    S15 & 100 & 50  & 50  & $46.7 \pm 16.4$ & 42/58   & 230 \\
    S19 & 87  & 36  & 51  & $37.3 \pm 11.0$ & 32/55   & 230 \\
    S20 & 533 & 251 & 282 & $39.2 \pm 14.7$ & 186/347 & 232 \\
    S21 & 156 & 70  & 86  & $35.3 \pm 12.6$ & 69/87   & 230 \\
    S25 & 152 & 63  & 89  & $67.3 \pm 6.7$  & 50/102  & 230 \\
    S7  & 87  & 49  & 38  & $41.9 \pm 12.6$ & 34/53   & 174 \\
    S8  & 150 & 75  & 75  & $28.6 \pm 10.9$ & 57/93   & 190 \\
    S9  & 100 & 50  & 50  & $28.5 \pm 8.6$  & 54/46   & 190 \\
    \midrule
    \textbf{Total} & \textbf{1596} & \textbf{751} & \textbf{845} &  & \textbf{629/967} &  \\
    \bottomrule
    \end{tabular}%
    }
    \vspace{-0.3cm}
    \endgroup
\end{table}

\begin{table}[t]
    \centering
    \caption{Statistics of the SRPBS dataset. $T$ denotes the number of
    fMRI time points. For institutions with multiple acquisition batches,
    \texttt{batch1} and \texttt{batch2} denote distinct, non-overlapping
    subject cohorts, and each batch is treated as an independent environment.}
    \vspace{0.2cm}
    \label{tab:srpbs_data}
    \begingroup
    \setlength{\tabcolsep}{4pt}
    \resizebox{0.8\textwidth}{!}{%
    \begin{tabular}{lrrrccr}
    \toprule
    Environment & $N$ & HC & MDD & Age (mean $\pm$ std) & Male/Female & $T$ \\
    \midrule
    HUH-batch1 & 114 & 57 & 57 & $38.8 \pm 13.1$ & 56/58 & 143 \\
    HUH-batch2 & 146 & 73 & 73 & $40.2 \pm 12.9$ & 63/83 & 143 \\

    HRC        & 32  & 16 & 16 & $41.5 \pm 12.0$ & 10/22 & 143 \\

    HKH-batch1 & 56  & 28 & 28 & $45.9 \pm 9.7$  & 29/27 & 107 \\
    HKH-batch2 & 58  & 29 & 29 & $44.7 \pm 9.9$  & 29/29 & 107 \\

    COI-batch1 & 140 & 70 & 70 & $48.3 \pm 13.4$ & 57/83 & 240 \\
    COI-batch2 & 142 & 71 & 71 & $48.3 \pm 12.7$ & 57/85 & 240 \\

    KUT        & 32  & 16 & 16 & $40.9 \pm 13.7$ & 16/16 & 240 \\

    UTO-batch1 & 124 & 62 & 62 & $43.1 \pm 14.4$ & 54/70 & 240 \\
    UTO-batch2 & 124 & 62 & 62 & $36.3 \pm 13.7$ & 66/58 & 240 \\
    \midrule
    \textbf{Total}
    & \textbf{968}
    & \textbf{484}
    & \textbf{484}
    & 
    & \textbf{437/531}
    & \\
    \bottomrule
    \end{tabular}%
    }
    \endgroup
    \vspace{-0.1cm}
\end{table}

\begin{table}[t]
    \centering
    \caption{Statistics of the ABCD dataset. $T$ denotes the number of fMRI time points.}
    \vspace{0.2cm}
    \label{tab:abcd_data}
    \begingroup
    \setlength{\tabcolsep}{4pt}
    \resizebox{0.6\textwidth}{!}{%
    \begin{tabular}{lrrrccr}
    \toprule
    Site & $N$ & HC & ADHD & Age (mean $\pm$ std) & Male/Female & $T$ \\
    \midrule
    site02 & 20  & 8   & 12  & $9.40 \pm 0.49$ & 13/7  & 383 \\
    site06 & 53  & 23  & 30  & $9.43 \pm 0.50$ & 30/23 & 383 \\
    site07 & 54  & 26  & 28  & $9.48 \pm 0.50$ & 34/20 & 383 \\
    site09 & 32  & 16  & 16  & $9.41 \pm 0.49$ & 20/12 & 383 \\
    site11 & 53  & 19  & 34  & $9.32 \pm 0.47$ & 29/24 & 383 \\
    site14 & 52  & 34  & 18  & $9.67 \pm 0.46$ & 22/30 & 383 \\
    site15 & 27  & 8   & 19  & $9.30 \pm 0.46$ & 16/11 & 383 \\
    site16 & 60  & 41  & 19  & $9.40 \pm 0.49$ & 33/27 & 383 \\
    site20 & 32  & 15  & 17  & $9.69 \pm 0.46$ & 17/15 & 383 \\
    site21 & 42  & 22  & 20  & $9.52 \pm 0.50$ & 32/10 & 383 \\
    \midrule
    \textbf{Total} & \textbf{425} & \textbf{212} & \textbf{213} &  & \textbf{246/179} &  \\
    \bottomrule
    \end{tabular}%
    }
    \vspace{-0.5cm}
    \endgroup
\end{table}

%% file: table/algorithm.tex
\begin{algorithm}[t]
\caption{Training and inference of \method{}}
\label{alg:overall}
\begin{algorithmic}[1]
\REQUIRE Source datasets
$\{\widehat{\mathcal D}^{e}\}_{e\in\mathcal E_{\mathrm{tr}}}$,
test inputs $\{(\bm X_i,\bm q_i)\}_{i=1}^{N_{\mathrm{test}}}$,
temporal parameters $(W,S_w)$,
scaffold parameters $(\tau_E,\eta_E,\zeta_E,B)$,
gate parameters $(\lambda,\tau,K)$,
loss weight $\beta$, and GNN depth $L$.
\ENSURE Test predictions $\widehat{\bm Y}_{\mathrm{test}}$.

\STATE Define $\widehat{\mathcal D}_{\mathrm{tr}}
=\bigcup_{e\in\mathcal E_{\mathrm{tr}}}\widehat{\mathcal D}^{e}$.

\STATE \textbf{Stage 1: Site-Aware Confounder Decoupling and Scaffold Extraction}
\STATE Estimate covariate means $\bm\mu_q$ and standard deviations
$\bm\sigma_q$ from $\widehat{\mathcal D}_{\mathrm{tr}}$.
\FOR{each $e\in\mathcal E_{\mathrm{tr}}$}
    \STATE Compute full-scan static FC vectors $\bm r_i$ and
    standardized covariates
    $\tilde{\bm q}_i=(\bm q_i-\bm\mu_q)/\bm\sigma_q$
    for subjects at site $e$.
    \STATE Estimate $(\hat{\bm b}_e,\hat{\bm\Gamma}_e,\hat{\bm\theta}_e)$
    by Huber regression (Eq.~\eqref{eq:7});
    $\hat{\bm\theta}_e$ is not subtracted.
    \STATE Obtain $\bm r_i^{\mathrm{site}}$
    (Eq.~\eqref{eq:site_res}).
    \STATE Compute the site-level contrast $\bm d_e$
    from class-wise robust means of $\bm r_i^{\mathrm{site}}$.
\ENDFOR
\STATE Compute $\bm d_{\mathrm{com}}$ as the component-wise median
of $\{\bm d_e\}_{e\in\mathcal E_{\mathrm{tr}}}$.
\STATE Compute $\kappa_j$ and estimate $\pi_j$ using
$B$ source-site bootstrap replicates (Eq.~\eqref{eq:11}).
\STATE Extract the population scaffold $\mathcal S$
(Eq.~\eqref{eq:12}).
\STATE Compute disagreement scores $u_e$ and reliability weights
$\alpha_e$
(Eq.~\eqref{eq:source_disagreement}).
\STATE Form the weighted deconfounder
$(\bar{\bm b}_w,\bar{\bm\Gamma}_w)$
(Eq.~\eqref{eq:weighted_deconfounder}).

\STATE \textbf{Stage 2: Transient Pathway Profiling}
\STATE Fix an enumeration $(j_p)_{p=1}^{M_{\mathcal S}}$
of $\mathcal S$, where $M_{\mathcal S}=|\mathcal S|$.
\FOR{each training subject $i$}
    \STATE Compute windowed correlations $\rho^{(t)}_{i,j}$
    and Fisher-$z$ values $c^{(t)}_{i,j}$ for
    $j\in\mathcal S$ using $(W,S_w)$.
    \STATE Compute temporal descriptors $s_{i,j}$ and $f_{i,j}$
    (Eq.~\eqref{eq:descriptor}).
    \STATE Construct $\bm h^{(0)}_{i,p}$ from
    $\bm r_i^{\mathrm{site}}$ and the temporal descriptors
    (Eq.~\eqref{eq:node_feature}).
\ENDFOR
\STATE Compute prior scores $\bm\mu_0$, line-graph adjacency
$\bm A$, and propagation matrix $\tilde{\bm A}$.

\STATE \textbf{Stage 3: Prior-Guided Subject-Adaptive Gating and Training}
\STATE Initialize trainable GNN and MLP parameters;
keep source-derived statistics, $\mathcal S$, $\bm\mu_0$,
and $\tilde{\bm A}$ fixed.
\WHILE{not converged}
    \STATE Sample a batch
    $\mathcal B\subset\widehat{\mathcal D}_{\mathrm{tr}}$.
    \FOR{each subject $i$ in $\mathcal B$}
        \STATE Compute $\bm c_i$, $g_{i,p}$, and
        $\bm h^{(L)}_{i,p}$
        (Eqs.~\eqref{eq:16}--\eqref{eq:graph_update}).
        \STATE Obtain $\bm z_i$ by gated readout
        (Eq.~\eqref{eq:graph_level}) and predict
        $\hat y_i=\rho(\bm z_i)$.
    \ENDFOR
    \STATE Update trainable parameters using the batch loss
    $\mathcal L=\mathcal L_{\mathrm{cls}}
    +\beta\mathcal L_{\mathrm{budget}}$
    (Eq.~\eqref{eq:final}).
\ENDWHILE
\STATE Select the parameters with the best validation score, where validation subjects are processed as in Stage 4.
\STATE \textbf{Stage 4: Inference on Held-out Sites}
\FOR{each test subject $i=1,\ldots,N_{\mathrm{test}}$}
    \STATE Compute $\bm r_i$ from $\bm X_i$ and standardize
    $\bm q_i$ using the source statistics
    $(\bm\mu_q,\bm\sigma_q)$.
    \STATE Obtain $\bm r_i^{\mathrm{res}}$ using the frozen
    deconfounder $(\bar{\bm b}_w,\bar{\bm\Gamma}_w)$ (Eq.~\eqref{eq:weighted_deconfounder}).
    \STATE Compute temporal descriptors on $\mathcal S$
    using $(W,S_w)$ and construct $\bm h^{(0)}_{i,p}$
    with $\bm r_i^{\mathrm{res}}$
    (Eq.~\eqref{eq:node_feature}).
    \STATE Compute $\bm c_i$, $g_{i,p}$, and
    $\bm h^{(L)}_{i,p}$ using the trained model and
    fixed $(\bm\mu_0,\tilde{\bm A})$.
    \STATE Obtain $\bm z_i$ by gated readout and predict
    $\hat y_i=\rho(\bm z_i)$.
\ENDFOR
\RETURN $\widehat{\bm Y}_{\mathrm{test}}
=(\hat y_i)_{i=1}^{N_{\mathrm{test}}}$.
\end{algorithmic}
\end{algorithm}

%% file: table/ID.tex
\begin{table*}[t]
    \centering
    \caption{Results on ABIDE, ABIDE (CC200), REST-meta-MDD, SRPBS, and ABCD (ADHD-task) under in-distribution (ID) settings. \textbf{Bold} results indicate the best performance, \underline{underlined} indicates second best.}
    \vspace{0.1cm}
    \label{tab:main_results_id}
    \resizebox{\textwidth}{!}{
    \begin{tabular}{cc |cc |cc |cc |cc |cc}
    \toprule
    \multirow{2}{*}{Type} & \multirow{2}{*}{Method}
    & \multicolumn{2}{c|}{ABIDE}
    & \multicolumn{2}{c|}{ABIDE (CC200)}
    & \multicolumn{2}{c|}{REST-meta-MDD}
    & \multicolumn{2}{c|}{SRPBS}
    & \multicolumn{2}{c}{ABCD (ADHD-task)} \\
    &
    & AUC & ACC
    & AUC & ACC
    & AUC & ACC
    & AUC & ACC
    & AUC & ACC \\
    \midrule

    % ==================== General GNNs ====================
    \multirow{3}{*}{
        \rotatebox{90}{\fontsize{9}{12}\selectfont GNN}
    }
    & GCN
    & 58.90$_{\pm 9.10}$ & 55.80$_{\pm 10.00}$
    & 48.34$_{\pm 9.12}$ & 51.05$_{\pm 7.95}$
    & 61.91$_{\pm 3.44}$ & 56.41$_{\pm 2.19}$
    & 83.23$_{\pm 4.21}$ & 74.87$_{\pm 5.19}$
    & 70.13$_{\pm 7.51}$ & 64.45$_{\pm 7.47}$ \\

    & GAT
    & 60.10$_{\pm 6.80}$ & 58.90$_{\pm 3.70}$
    & 59.12$_{\pm 9.54}$ & 55.88$_{\pm 8.41}$
    & 62.91$_{\pm 3.44}$ & 58.60$_{\pm 2.89}$
    & 81.02$_{\pm 4.87}$ & 69.58$_{\pm 5.42}$
    & 70.65$_{\pm 5.77}$ & 62.31$_{\pm 6.25}$ \\

    & GIN
    & 58.50$_{\pm 5.40}$ & 57.50$_{\pm 5.60}$
    & 53.76$_{\pm 8.21}$ & 48.91$_{\pm 11.23}$
    & 58.11$_{\pm 2.97}$ & 55.41$_{\pm 1.82}$
    & 81.62$_{\pm 4.77}$ & 70.39$_{\pm 5.61}$
    & 72.27$_{\pm 6.01}$ & 63.25$_{\pm 5.97}$ \\

    \midrule

    % ==================== General OOD ====================
    \multirow{2}{*}{
        \rotatebox{90}{\fontsize{9}{12}\selectfont OOD}
    }
    & CORAL
    & 55.90$_{\pm 6.60}$ & 55.10$_{\pm 6.10}$
    & 55.67$_{\pm 9.38}$ & 50.21$_{\pm 8.76}$
    & 58.85$_{\pm 3.52}$ & 56.07$_{\pm 3.46}$
    & 79.82$_{\pm 4.71}$ & 69.33$_{\pm 3.86}$
    & 70.96$_{\pm 6.80}$ & 61.66$_{\pm 7.12}$ \\

    & IRM
    & 55.40$_{\pm 7.80}$ & 54.20$_{\pm 6.10}$
    & 50.12$_{\pm 11.45}$ & 54.89$_{\pm 6.32}$
    & 59.62$_{\pm 2.57}$ & 56.26$_{\pm 3.08}$
    & 79.14$_{\pm 4.43}$ & 69.19$_{\pm 4.52}$
    & 70.73$_{\pm 6.70}$ & 63.98$_{\pm 8.24}$ \\

    \midrule

    % ==================== Graph OOD ====================
    \multirow{4}{*}{
        \rotatebox{90}{
            \shortstack{\fontsize{9}{12}\selectfont Graph\\OOD}
        }
    }
    & GSAT
    & 56.20$_{\pm 4.90}$ & 53.80$_{\pm 6.30}$
    & 49.88$_{\pm 7.82}$ & 51.34$_{\pm 6.05}$
    & 57.65$_{\pm 3.84}$ & 54.35$_{\pm 3.97}$
    & 81.69$_{\pm 3.25}$ & 70.87$_{\pm 4.62}$
    & 70.92$_{\pm 7.76}$ & 63.78$_{\pm 6.62}$ \\

    & DisC
    & 58.80$_{\pm 7.20}$ & 54.40$_{\pm 7.30}$
    & 54.19$_{\pm 6.22}$ & 52.11$_{\pm 6.87}$
    & 58.75$_{\pm 4.21}$ & 56.44$_{\pm 3.17}$
    & 81.91$_{\pm 4.39}$ & 72.46$_{\pm 3.83}$
    & 72.27$_{\pm 6.08}$ & 63.54$_{\pm 6.60}$ \\

    & CEPG
    & 50.19$_{\pm 6.63}$ & 50.57$_{\pm 7.38}$
    & 44.52$_{\pm 10.15}$ & 49.33$_{\pm 9.02}$
    & 56.36$_{\pm 4.72}$ & 54.26$_{\pm 4.15}$
    & 74.67$_{\pm 8.22}$ & 67.74$_{\pm 5.83}$
    & 71.24$_{\pm 8.20}$ & 63.74$_{\pm 9.19}$ \\

    & DiSCO
    & 52.27$_{\pm 6.19}$ & 51.72$_{\pm 5.39}$
    & 51.98$_{\pm 8.65}$ & 50.04$_{\pm 7.21}$
    & 59.41$_{\pm 4.86}$ & 55.77$_{\pm 2.87}$
    & 84.14$_{\pm 3.26}$ & 72.78$_{\pm 4.39}$
    & 68.75$_{\pm 9.15}$ & 63.14$_{\pm 8.33}$ \\

    \midrule

    % ==================== Brain Networks ====================
    \multirow{9}{*}{
        \rotatebox{90}{
            \shortstack{\fontsize{9}{12}\selectfont Brain\\Networks}
        }
    }

    % 2021
    & STAGIN
    & 59.42$_{\pm 6.15}$ & 57.38$_{\pm 5.24}$
    & 58.15$_{\pm 7.22}$ & 56.41$_{\pm 6.89}$
    & 59.18$_{\pm 4.25}$ & 55.72$_{\pm 3.41}$
    & 81.33$_{\pm 5.12}$ & 71.45$_{\pm 4.98}$
    & 69.21$_{\pm 6.84}$ & 60.15$_{\pm 6.11}$ \\

    % 2022
    & BrainNetTF
    & 60.60$_{\pm 7.00}$ & 58.90$_{\pm 3.70}$
    & 58.91$_{\pm 10.22}$ & 51.45$_{\pm 8.91}$
    & 61.70$_{\pm 2.40}$ & 56.92$_{\pm 2.31}$
    & 78.43$_{\pm 7.78}$ & 69.39$_{\pm 6.61}$
    & 70.08$_{\pm 7.30}$ & 58.34$_{\pm 7.89}$ \\

    % 2024
    & XG-GNN
    & 52.61$_{\pm 7.51}$ & 53.08$_{\pm 3.04}$
    & 49.25$_{\pm 13.02}$ & 55.31$_{\pm 8.44}$
    & 59.31$_{\pm 3.99}$ & 55.33$_{\pm 3.17}$
    & 83.43$_{\pm 7.34}$ & 72.21$_{\pm 8.26}$
    & 71.32$_{\pm 8.35}$ & 56.90$_{\pm 5.94}$ \\

    % 2025
    & AGMGC
    & 58.04$_{\pm 7.33}$ & 52.41$_{\pm 6.02}$
    & 54.83$_{\pm 7.94}$ & 52.88$_{\pm 5.56}$
    & 57.78$_{\pm 5.51}$ & 53.88$_{\pm 1.75}$
    & 85.64$_{\pm 3.91}$ & 75.83$_{\pm 3.15}$
    & 69.75$_{\pm 6.99}$ & 61.65$_{\pm 7.85}$ \\

    & FC-HGNN
    & 55.43$_{\pm 11.85}$ & 51.74$_{\pm 5.76}$
    & 56.77$_{\pm 9.81}$ & 54.12$_{\pm 5.99}$
    & 60.27$_{\pm 5.67}$ & 55.95$_{\pm 4.12}$
    & 76.07$_{\pm 5.48}$ & 67.22$_{\pm 6.21}$
    & \underline{72.69$_{\pm 7.92}$} & 61.68$_{\pm 9.03}$ \\

    & LHDFormer
    & 60.85$_{\pm 5.41}$ & 57.92$_{\pm 4.77}$
    & 61.34$_{\pm 5.18}$ & 58.05$_{\pm 5.02}$
    & 62.15$_{\pm 3.86}$ & 57.84$_{\pm 2.95}$
    & 84.72$_{\pm 4.35}$ & 74.55$_{\pm 3.82}$
    & 71.84$_{\pm 6.05}$ & 62.91$_{\pm 5.44}$ \\

    & NeuroH-TGL
    & 59.76$_{\pm 7.12}$ & 58.11$_{\pm 5.88}$
    & 60.52$_{\pm 6.45}$ & 57.29$_{\pm 5.61}$
    & 61.42$_{\pm 3.92}$ & 58.21$_{\pm 3.04}$
    & 83.91$_{\pm 4.67}$ & 73.18$_{\pm 4.15}$
    & 70.93$_{\pm 6.32}$ & 61.85$_{\pm 6.72}$ \\

    & BrainOOD
    & \underline{61.29$_{\pm 8.10}$} & 58.19$_{\pm 6.64}$
    & \underline{62.45$_{\pm 4.95}$} & 59.08$_{\pm 4.25}$
    & \underline{63.49$_{\pm 4.01}$} & \underline{60.01$_{\pm 3.11}$}
    & 85.11$_{\pm 4.02}$ & 75.92$_{\pm 4.14}$
    & 72.17$_{\pm 6.46}$ & \underline{64.92$_{\pm 5.83}$} \\

    % 2026
    & DeCI
    & 59.89$_{\pm 4.63}$ & \textbf{62.07$_{\pm 3.65}$}
    & 57.94$_{\pm 4.52}$ & \textbf{62.27$_{\pm 3.81}$}
    & 53.85$_{\pm 2.94}$ & 54.39$_{\pm 2.50}$
    & \underline{86.63$_{\pm 5.71}$} & \underline{78.29$_{\pm 4.68}$}
    & 60.00$_{\pm 7.86}$ & 60.72$_{\pm 4.03}$ \\

    \midrule

    % ==================== Ours ====================
    & \method{}
    & \textbf{63.35$_{\pm 7.82}$} & \underline{60.12$_{\pm 6.05}$}
    & \textbf{64.91$_{\pm 7.42}$} & \underline{61.52$_{\pm 4.51}$}
    & \textbf{64.16$_{\pm 3.12}$} & \textbf{60.41$_{\pm 2.85}$}
    & \textbf{87.23$_{\pm 3.01}$} & \textbf{79.58$_{\pm 2.45}$}
    & \textbf{73.05$_{\pm 6.12}$} & \textbf{65.94$_{\pm 6.33}$} \\

    \bottomrule
    \end{tabular}
    }
    \vspace{-0.2cm}
\end{table*}

%% file: table/top10_edge.tex
\begin{table}[t]
    \centering
    \caption{Top 10 scaffold edges with positive aggregate diagnostic contrasts on ABIDE under the OOD setting. Edge strength is measured by $|d_{\mathrm{com},j}|$, the magnitude of the consensus contrast averaged over the 100 fold--run evaluations. $\nu_j$ denotes selection frequency across evaluations, and $\kappa_j$ denotes source-site sign consistency within each evaluation, averaged across evaluations.}
    \vspace{0.2cm}
    \label{tab:avg_edges}
    \resizebox{\textwidth}{!}{%
    \begin{tabular}{rllccccc}
    \toprule
    Rank & ROI 1 & ROI 2 & Sign & $|d_{\mathrm{com},j}|$ & Gate & $\nu_j$ & $\kappa_j$ \\
    \midrule
    1  & Parietal\_Inf\_R        & Angular\_L              & + & 0.0757 & 1.000 & 1.00 & 0.92 \\
    2  & Thalamus\_L             & Temporal\_Mid\_L        & + & 0.0715 & 0.976 & 0.99 & 0.89 \\
    3  & Frontal\_Mid\_R         & Angular\_L              & + & 0.0632 & 0.984 & 1.00 & 0.94 \\
    4  & Cingulum\_Mid\_L        & Occipital\_Sup\_R       & + & 0.0612 & 0.978 & 0.98 & 0.88 \\
    5  & Thalamus\_L             & Temporal\_Sup\_L        & + & 0.0571 & 0.775 & 0.94 & 0.86 \\
    6  & Paracentral\_Lobule\_L  & Pallidum\_R             & + & 0.0526 & 0.979 & 1.00 & 0.89 \\
    7  & Parietal\_Inf\_R        & Angular\_R              & + & 0.0518 & 0.869 & 0.96 & 0.90 \\
    8  & Lingual\_R              & Putamen\_R              & + & 0.0498 & 0.771 & 0.91 & 0.84 \\
    9  & Pallidum\_L             & Temporal\_Mid\_R        & + & 0.0476 & 0.852 & 0.99 & 0.87 \\
    10 & Insula\_L               & Paracentral\_Lobule\_L  & + & 0.0472 & 0.857 & 0.97 & 0.88 \\
    \bottomrule
    \end{tabular}%
    }
    \vspace{-0.5cm}
\end{table}